\documentclass[preprint,10pt]{elsarticle}
\usepackage[letterpaper,top=2.75cm,bottom=2.75cm,left=2.75cm,right=2.75cm,marginparwidth=2.75cm]{geometry}

\usepackage[T1]{fontenc}
\usepackage[utf8]{inputenc} 
\usepackage{microtype}
\usepackage{setspace}
	
\usepackage{nicefrac} 

\usepackage{siunitx} 
\usepackage[letterpaper,top=2.75cm,bottom=2.75cm,left=2.75cm,right=2.75cm,marginparwidth=2.75cm]{geometry}

\usepackage{mathtools}  
\usepackage{amsthm,amsfonts,bm}

\newtheorem{definition}{Definition}

\newtheorem{theorem}{Theorem}

\usepackage{graphicx,booktabs,multirow,dcolumn,subcaption}
\usepackage[ruled,vlined]{algorithm2e}
    \DontPrintSemicolon
\usepackage{listings}

\usepackage[backref=page]{hyperref}
\usepackage[nameinlink,capitalize]{cleveref}

\usepackage{xcolor}
\newcommand{\reva}[1]{\textcolor{black}{#1}} 
\newcommand{\revb}[1]{\textcolor{black}{#1}} 

\newcommand{\OO}{\mathcal{O}}

\newcommand{\reals}{\mathbb{R}}

\usepackage{cleveref}

\DeclareSIUnit{\ttimes}{$\times$}

\journal{Elsevier}

\newcommand{\inr}[0]{\bm{\Phi}}

\begin{document}

    \begin{frontmatter}
    
        \title{\textit{In Situ} Training of Implicit Neural Compressors for Scientific Simulations via Sketch-Based Regularization}
        
        \author[1,3]{Cooper Simpson\corref{mycorrespondingauthor}}
        \ead{rscooper@uw.edu}
        \author[2]{Stephen Becker}
        \ead{stephen.becker@colorado.edu}
        \author[3]{Alireza Doostan\corref{mycorrespondingauthor}}
        \cortext[mycorrespondingauthor]{Corresponding author}
        \ead{doostan@colorado.edu}
        \affiliation[1]{address={Applied Mathematics, University of Washington, Seattle}}
        \affiliation[2]{address={Applied Mathematics, University of Colorado, Boulder}}
        \affiliation[3]{address={Ann and H.J. Smead Department of Aerospace Engineering Sciences, University of Colorado, Boulder}}
        
        \begin{abstract}
            Focusing on implicit neural representations, we present a novel \textit{in situ} training protocol that employs limited memory buffers of full and sketched data samples, where the sketched data are leveraged to prevent catastrophic forgetting. The theoretical motivation for our use of sketching as a regularizer is presented via a simple Johnson-Lindenstrauss-informed result. While our methods may be of wider interest in the field of \textit{continual learning}, we specifically target \textit{in situ} neural compression using implicit neural representation-based hypernetworks. We evaluate our method on a variety of complex simulation data in two and three dimensions, over long time horizons, and across unstructured grids and non-Cartesian geometries. On these tasks, we show strong reconstruction performance at high compression rates. Most importantly, we demonstrate that sketching enables the presented \textit{in situ} scheme to approximately match the performance of the equivalent offline method.
        \end{abstract}
        
        \begin{keyword}
            Continual learning \sep Neural fields \sep Streaming data \sep Scientific machine learning \sep Unstructured data \sep Catastrophic forgetting
        \end{keyword}
    
    \end{frontmatter}


    \section{Introduction}\label{sec:intro}

        Modern large-scale scientific simulations can generate enormous datasets that are far too large to store in their entirety offline, yet one would still like access to the data for analysis. This necessitates performant compression methods, i.e., those capable of high compression rates greater than 100\(\times\) and low relative errors, e.g., close to 1\%. Of equal importance is for these methods to be operable \textit{in situ} with the simulation itself. This is because, by definition, the problem assumes there is no offline access to the relevant data, so any practical application of the compression tool must be \textit{in situ}. Another difficult aspect of compressing scientific data is that the underlying meshes are often not uniformly structured. Many simulations require non-Cartesian domains and/or may alter the density of nodes throughout this domain to resolve specific phenomena accurately, requiring compression techniques that can operate on an arbitrary collection of points. \Cref{fig:mesh-types} details examples of commonly-used uniform/non-uniform and structured/unstructured meshes. Our work seeks to address these three challenging problems.

        \begin{figure}[th]
            \centering
            \includegraphics[trim = 410mm 120mm 410mm 120mm, clip, width=0.24\textwidth]{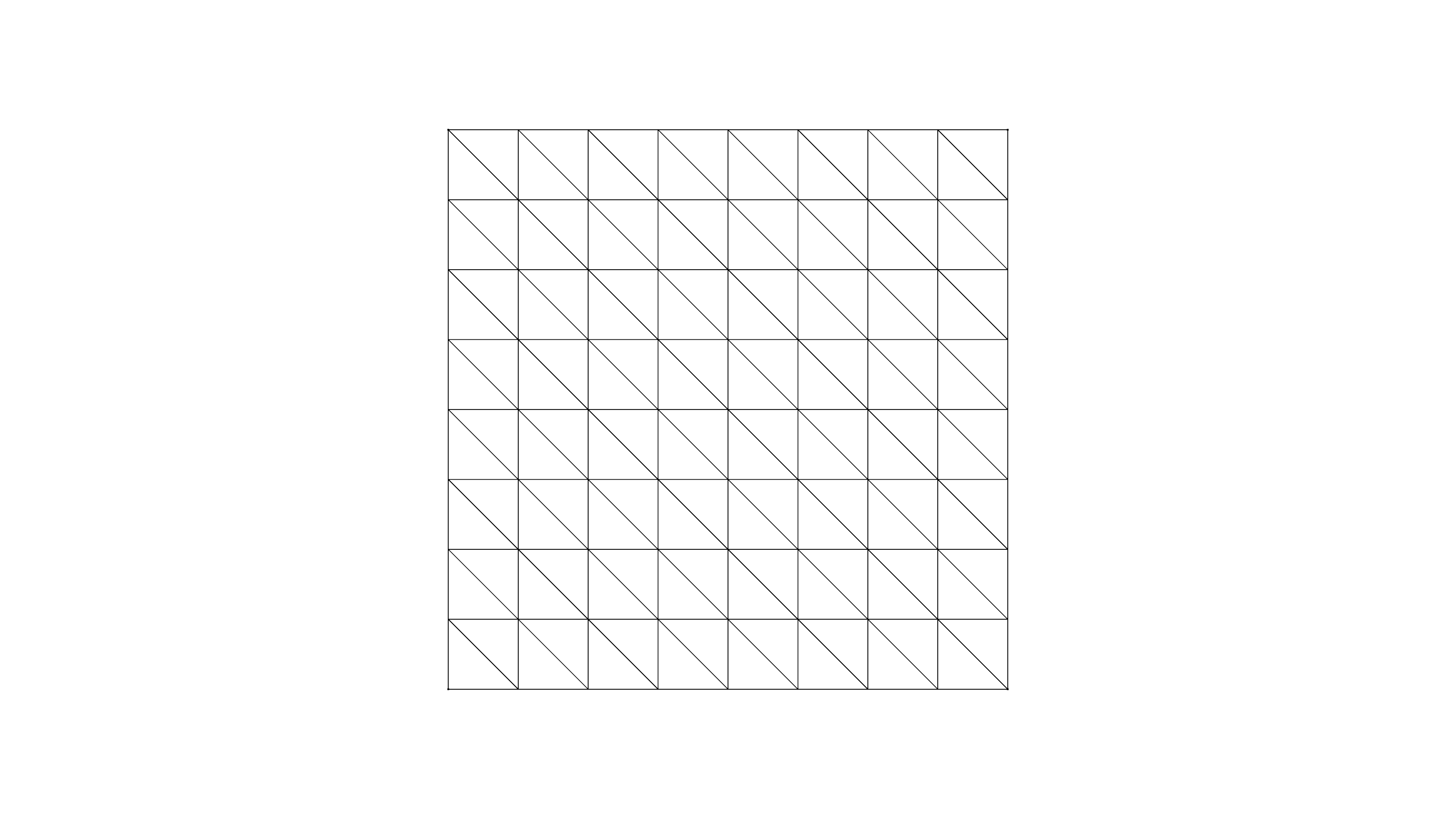}
            \includegraphics[trim = 410mm 120mm 410mm 120mm, clip, width=0.24\textwidth]{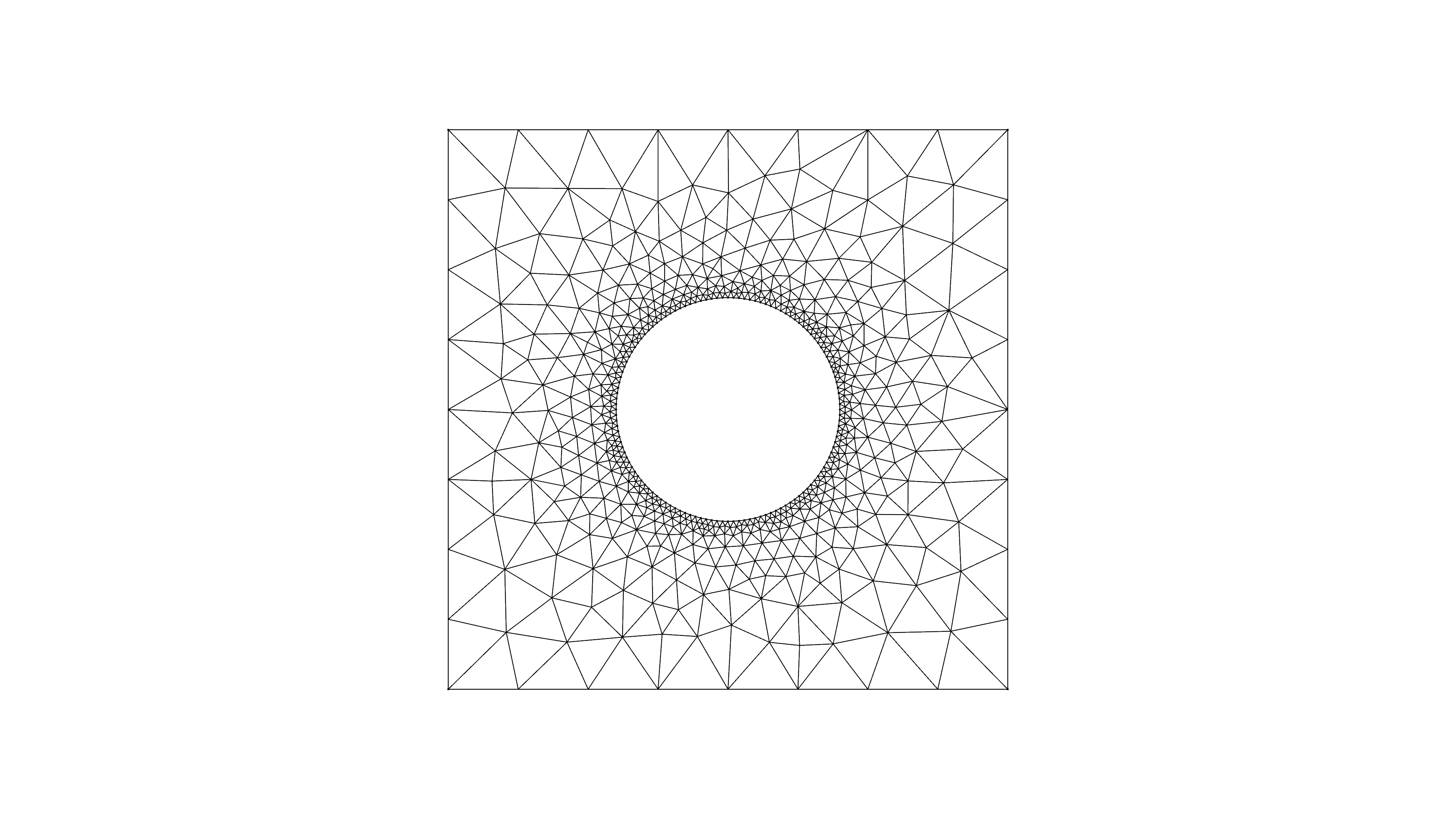}
            \includegraphics[width=0.49\textwidth]{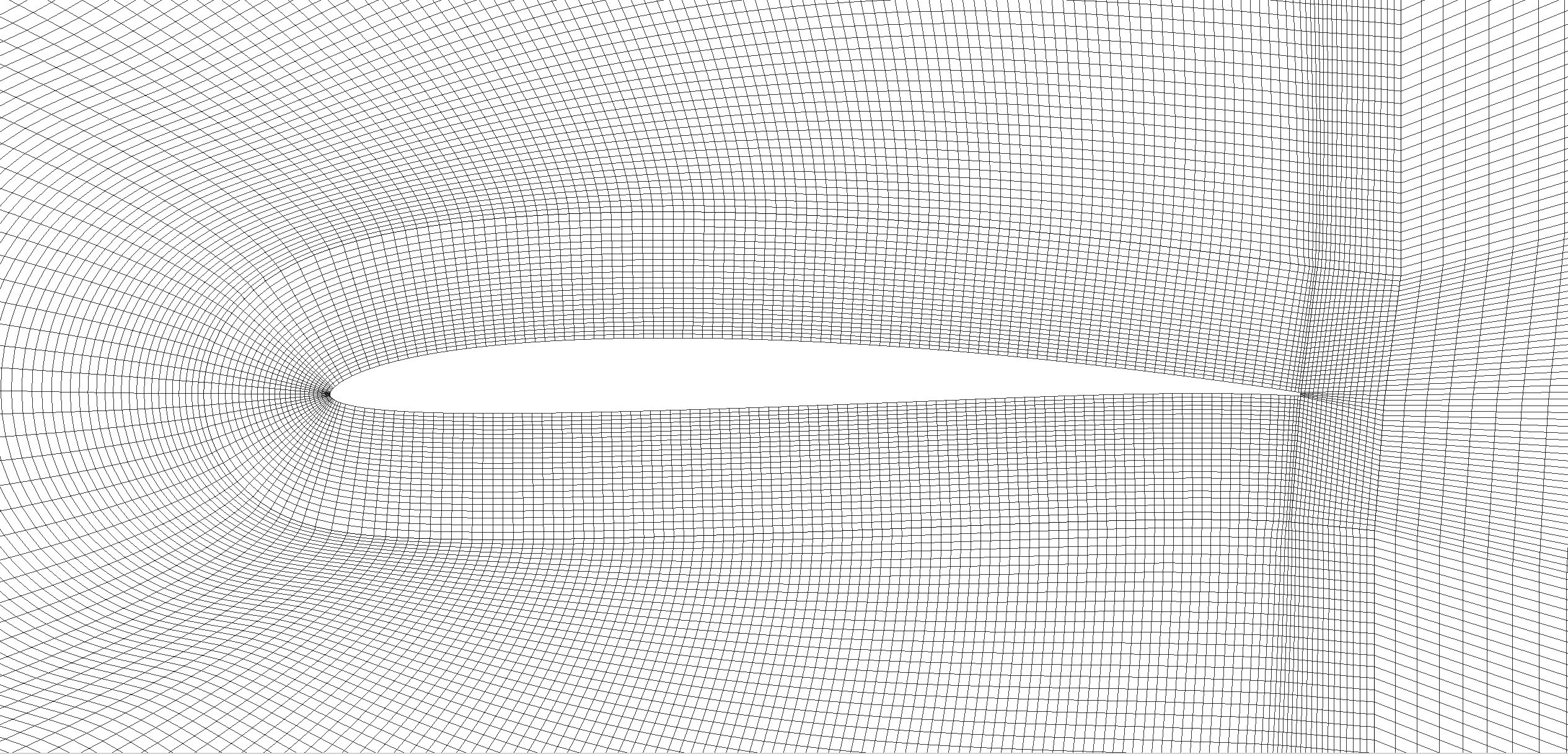}          
            \caption{The compression approach of this study applies to all mesh/geometry types, including (uniform) structured (left), unstructured (middle), and curvilinear (right), as it uses the node locations without connectivity information.}
            \label{fig:mesh-types}
        \end{figure}

         Neural network-based compression methods can capture non-linear phenomena and sharp features, such as physical discontinuities, that pose difficulties for traditional linear dimensionality reduction techniques, such as those based on singular value decomposition (SVD). Several neural compression approaches exist, which can compress the data directly or learn a reduced dimensionality representation, but all require optimizing neural network weights under some training strategy. One such method, implicit neural representation (INR) \cite{park2019deepsdf,genova2019learning, sitzmann2020implicit}, is an implicit, continuous, and differentiable neural network model that, conveniently, only requires access to the space-time coordinates of points, circumventing the connectivity structure entirely. INR, with some specific recent developments, has shown outstanding accuracy in representing a diverse set of signals; see, e.g., \cite{sitzmann2020implicit,zhang2021implicit,pan2023neural,han2023kd,sales2024implicit} or \Cref{fig:inr-offline} wherein we illustrate sample compression results for INRs trained offline on two complex simulation datasets.
        
        \begin{figure}[b]
            \centering
            \begin{subfigure}[b]{0.48\textwidth}
                \includegraphics[width=\textwidth]{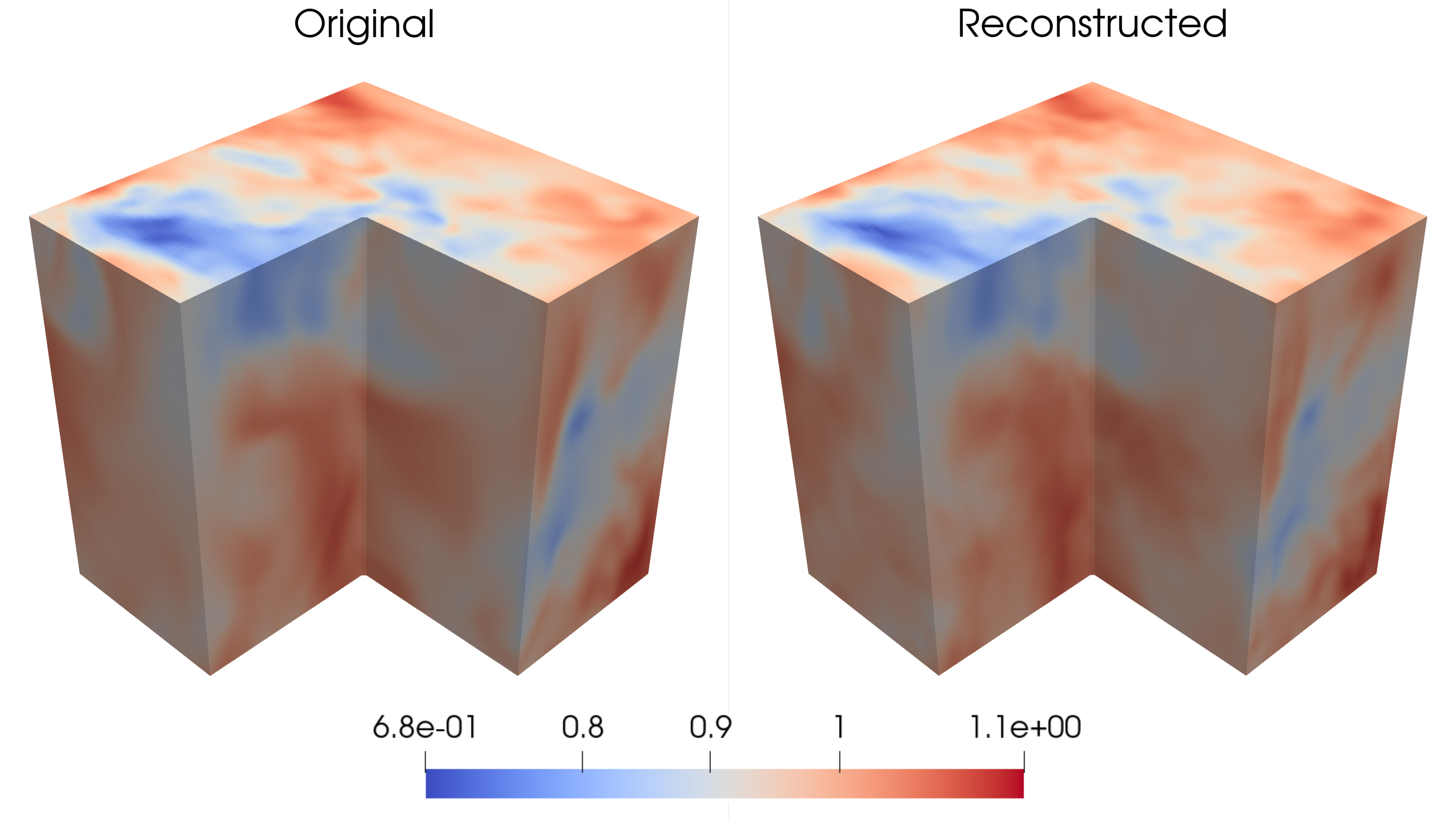}
                \caption{Streamwise velocity INR reconstruction at 1.1\% relative error and 40.5dB PSNR for a 2,366\(\times\) compression rate.}
                \label{fig:channel-flow-offline-x}
            \end{subfigure}
            \hfill
            \begin{subfigure}[b]{0.48\textwidth}
                \includegraphics[width=\textwidth]{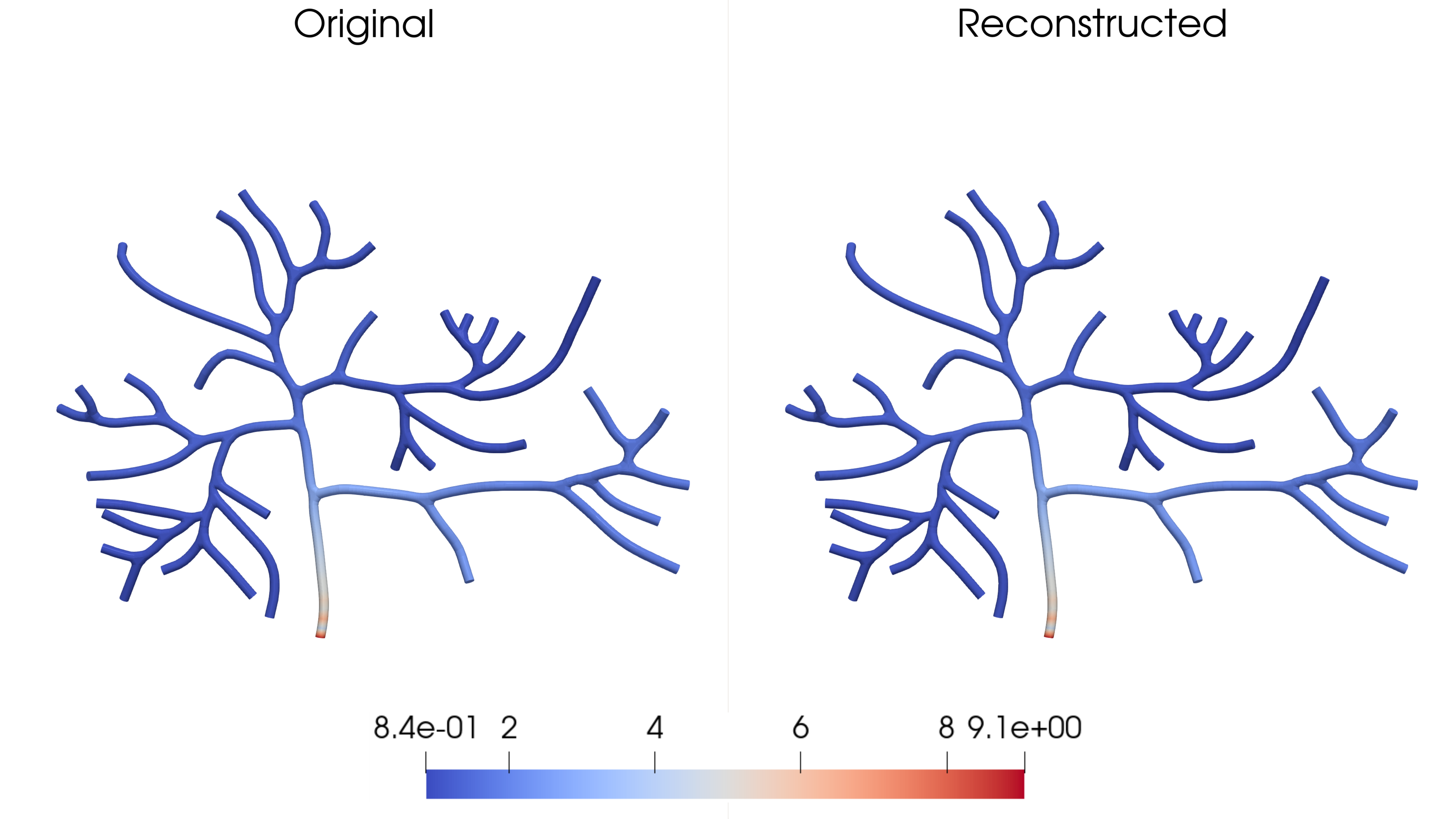}
                \caption{Species concentration (\%) INR reconstruction at 0.5\% relative error and 62.9dB PSNR for a 1,392\(\times\) compression rate.}
                \label{fig:neuron-transport-offline}
            \end{subfigure}
            \caption{Compression results for INRs trained offline. See \cref{sec:experiments} for details on the datasets.}
            \label{fig:inr-offline}
        \end{figure}
        
        Our interest is in extending these INR compression techniques to the \textit{in situ} setting, meaning the scenario where data samples (or points, entries, etc.) arrive one-at-a-time from a process (e.g., simulation) and cannot all be stored due to memory constraints. We are careful to use the term \textit{in situ} as opposed to online or streaming because our procedure will not be able to run indefinitely and will not end precisely when the data generation process ends.
        
        Available data samples are stored in what we will refer to as the \textit{buffer}, which in the extreme case is only large enough to store a single sample, though we will typically target the case when the buffer can hold the equivalent of a few dozen data samples. Altering the size of this buffer, along with varying how many samples arrive together or what information from the samples is stored, allows one to interpolate between the pure online and offline settings. Our notion of samples is general and need not be a single number; in a typical example, a data sample is the entire spatial field at a given time (a snapshot), arising from a time-dependent partial differential equation (PDE) simulation. Non-stationary training for traditional machine learning techniques has been widely studied, but often less so in the case of neural networks, where it is typically referred to as continual learning. Hypernetworks, \cite{ha2017hypernetworks}, which are neural network architectures that learn to generate weights for a target network, are perhaps the most successful and well-known method for continual learning. Appropriated from multi-task learning \cite{caruana1997multitask}, they allow for a natural extension to \textit{in situ} learning by equating one or multiple time steps with a task.

        In this study, we present a novel learning procedure based on randomized sketching for \textit{in situ} training of accurate and highly compressed neural representations of scientific simulation data using an INR-based hypernetwork architecture. Importantly, our methods are mesh-agnostic, only requiring access to nodes, so they are applicable to arbitrarily complex geometries. We highlight the fact that while sketching has been used extensively for low-rank compression, to our knowledge, this is the first time it has been used for regularizing neural compression. The remainder of this section will cover the necessary background and related work on INR, hypernetworks, and compression. \Cref{sec:insitu-training} will detail our \textit{in situ} training, and \cref{sec:experiments} extensively investigates our method's practical performance on a variety of relevant datasets.

        \subsection{Preliminaries and Related Work}\label{sec:preliminaries}

            We will denote vectors using lowercase bold font (e.g., \(\bm{x}\)) and matrices or operators using uppercase bold font (e.g., \(\bm{X}\)). We will specialize to the task of compressing PDE simulation data and consider a fixed, possibly unstructured or non-uniform, mesh represented as \(\bm{X}\in\reals^{n\times d}\) consisting of \(n\) points in \(d\) spatial dimensions (typically 2 or 3). Additionally, we are provided with a \(c\)-dimensional time-dependent signal (or vector-field) \(\bm{u}:\reals^{d+1}\to\reals^c\) sampled at all mesh nodes for time steps \(t=1,2,\ldots,T\), generated by the numerical PDE solver\footnote{For simplicity of notation, we assume a uniform time grid with unit spacing, but the method extends straighforwardly to nonuniform time steps.}. For \(c>1\), we may think of the different output coordinates of \(\bm{u}\) as \emph{channels}; for example, $\bm{u}$ might represent the 3D velocity field of a fluid, in which case $c=3$.

            Although our core methods are based on points \((\bm{x},t)\in\reals^{d+1}\), for training, we will often consider a single sample to be a snapshot, i.e., \(\bm{U}_t:=\bm{u}(\bm{X},t)\), \(\bm{U}_t\in\reals^{n\times c}\). We then define the reconstruction of this snapshot, via a neural network \(\inr:\reals^{d+1}\to \reals^c\) parameterized by \(\bm{\theta}\), as \(\bm{\widetilde{U}}_t:=\inr(\bm{X},t;\bm{\theta})\). The concatenation of these into the full \(T\times n\times c\) dateset and its reconstruction are given as \(\bm{U}=[\bm{U}_1;\cdots;\bm{U}_T]\) and \(\bm{\widetilde{U}}=[\bm{\widetilde{U}}_1;\cdots;\bm{\widetilde{U}}_T]\), respectively. \reva{In certain contexts, we will drop the dependence of \(\bm{\Phi}\) on one or more of its arguments to make clear that these are fixed. For example, we write \(\bm{\Phi}(\bm{\theta})\) when considering the network as a function of its parameters for fixed input coordinates.}
            
            Our method will be used to optimize the parameters \(\bm{\theta}\) of the neural network, whereupon denoting the number of parameters as \(|\bm{\theta}|\), we can define the compression rate as follows:
            \begin{equation}\label{eq:cr}
                \frac{T\times n\times c}{|\bm{\theta}|}.
            \end{equation}
            Note that this does not include the required storage for the mesh or network definition, or intermediate storage for the \textit{in situ} learning scheme. The formula assumes consistent floating-point representation of all numbers, but should be adjusted in the obvious way if, for example, some of the network weights are represented in different precisions.

            \paragraph{Implicit neural representation (INR)} INR is a straightforward form of supervised deep learning to train a neural network \(\inr\) to approximate a target signal \(\bm{f}\), and is the building block framework upon which popular extensions, such as PINNs \cite{raissi2019physics}, have been built. The key concept is to build a neural network that views input as discrete samples of a continuous object, so that dimensionality does not increase with resolution. Typically, the network in question is of a simple, densely-connected design, and the signal is a function of spatial or spatio-temporal coordinates (for notational simplicity, in this subsection we assume a purely spatial domain and hence denote the coordinates by \(\bm{x}\)). To be more precise, we consider a signal \(\bm{f}:\reals^d\to\reals^c\) that maps coordinates in \(d\) dimensions to a vector of \(c\) features. Examples include 2D image pixel coordinates mapping to RGB colors, or mesh coordinates and time mapping to the flow field of a numerical PDE solution. \Cref{eq:inr-network} defines the simplest INR with \(L\) feedforward layers and non-linear activation function \(\sigma\), but one should note that more complicated architectural features, such as residual connections, can be included:
            \begin{align}\label{eq:inr-network}
                &\bm{x}_0 = \bm{x}, \nonumber\\
                &\bm{x}_{\ell} := \bm{\phi}_{\ell}(\bm{x}_{{\ell}-1}) = \sigma\left(\bm{W}_{\ell}\bm{x}_{{\ell}-1}+\bm{b}_{\ell}\right), \quad(\forall \ell=1,\ldots,L-1), \nonumber\\
                &\inr(\bm{x}) = \bm{W}_L(\bm{\phi}_{L-1}\circ\bm{\phi}_{L-2}\circ\cdots\circ\bm{\phi}_{1})(\bm{x})+\bm{b}_L,
            \end{align}
            for some trainable weight matrices \(\bm{W}_{\ell}\) and biases \(\bm{b}_{\ell}\), which are collected into the set of network parameters \(\bm{\theta}\). A dataset consists of coordinate-value pairs \(\{(\bm{x}_i,\bm{f}_i=\bm{f}(\bm{x}_i))\}\), which one uses to optimize the parameters \(\bm{\theta}\) of \(\inr\), such that \(\mathcal{L}(\inr(\bm{x}_i;\bm{\theta}), \bm{f}_i)\) is minimized for some target loss \(\mathcal{L}\). It is not uncommon for the loss to be of Sobolev type, i.e., including a term targeting the difference of the derivatives (either exact or approximate) of \(\bm{f}\) and \(\inr\) with respect to some or all of the coordinates.
    
            The activation function \(\sigma\) has been shown to play a significant role in the performance of INR. Most notably, the work of \cite{sitzmann2020implicit} demonstrated that a sinusoidal activation with a specific random initialization of the network weights is a vastly superior option compared to ReLU, Tanh, or other popular alternatives, calling their framework a SIREN network. Extensions of the periodic activation function have been considered by, e.g., \cite{saragadam2023wire, roddenberry2023implicit}, which introduce and analyze complex wavelets. See the comprehensive review \cite{essakine2024we} for further details on the current state of INR.  
    
            \paragraph{Hypernetworks} Hypernetworks are a simple but powerful concept in deep learning wherein one network, the hypernetwork, is used to generate weights for a target network. Originally developed by \cite{ha2017hypernetworks} as a form of weight sharing for convolutional and recurrent models, they have since been integrated with a variety of deep learning pipelines. \revb{Often, the hypernetwork is much smaller than the target network\footnote{\revb{Although this is heavily dependent on the mapping and is not the case in our models.}}} and maps from an encoding of a specific task; for example, language pairs in translation. Less complicated architectures map directly to the full vector of target weights, but for large target models, this can be inefficient. Other approaches generate weights in chunks or via scaling factors. For details on the extensive work surrounding hypernetworks, see the review \cite{chauhan2024brief}.
    
            Of particular relevance to our work is that of \cite{pan2023neural}, which combines hypernetworks and INR for learning low-dimensional representations of large scientific datasets. We use a similar sine-based INR with residual connections and generate time-dependent weights in the hypernetwork. Compression is not the focus of their work, and all models are trained offline, unlike our focus on \textit{in situ} training and compression. Further, many of the stated disadvantages in the aforementioned work stem from a comparison to the SVD, which does not require hyperparameter tuning or network optimization. However, when the signal of interest exhibits a large Kolmogorov $n$-width, as in compressible and advection-dominated flows, linear dimensionality reduction techniques, such as the SVD, cease to be effective, thus motivating the use of non-linear dimensionality reduction, e.g., methods based on neural networks \cite{lee2020model}. 
            We also find that most of the remaining disadvantages, such as improving parameter efficiency or reducing training time, are deployment tasks, which can be tackled in isolation from the core compression problem.
    
            \paragraph{Non-stationary learning} The machine learning community has used many synonymous or related terms, such as online learning, incremental learning, and continual learning, to describe training on data with a non-stationary distribution. This setting stands in contrast to the usual offline regime of most deep learning, where one assumes unfettered access to a fixed training dataset, which is used to optimize a model that is then deployed or tested on other data. In reality, a spectrum exists between fully offline learning and fully online learning, where samples arrive once and one at a time. Possible variations include multiple samples available at one time or the ability to store some data offline. The key difficulty lies in avoiding catastrophic forgetting, i.e., maintaining good performance for samples seen earlier in the training process \cite{kirkpatrick2017overcoming}. 
                
            Hypernetworks are especially adept at tackling these non-stationary problems, as the work of \cite{von2019continual} highlights. To address catastrophic forgetting, they introduce a regularization term that seeks to maintain the hypernetwork output on previously learned tasks while learning a new task. A copy of the model is saved after learning each new task, which is used to reproduce the output of the hypernetwork on previous tasks. This method yields exceptional performance and limited forgetting, but the maximum number of tasks never exceeds 100. In our context, a task is a time step, and 100 is far fewer than the number of time steps often encountered in PDE simulations.
    
            Our regularization approach using sketched spatial snapshots is more similar to the {\it experience replay} or {\it rehearsal} techniques commonly seen in reinforcement learning \cite{rolnick2019experience} and throughout continual learning \cite{wang2024comprehensive}, though the technical details are notably different. 
    
            \paragraph{Compression} Both lossless and lossy compression for scientific data have long been targets of investigation. We emphasize the fact that this is a very dynamic field with significant amounts of research, much of it beyond our scope.
            
            A large portion of this effort has been focused on error-bounded schemes for highly accurate compression of structured data, but at relatively small compression rates. Well-known tools include ZFP \cite{lindstrom2014zfp} and SZ3 \cite{liang2023sz}, along with the recently introduced MFZ \cite{doherty2024mfz} that extends to non-uniform data. See the recent review \cite{di2025survey} for further details on error-bounded methods.
            
            As mentioned earlier, significant effort has been put towards linear low-rank approximation methods for compression, motivated by the best\footnote{In any unitarily invariant norm.} rank-\(k\) approximation from the SVD \cite{eckart1936approximation}. Deterministic methods are often too computationally expensive for large datasets, so many methods turn to randomized decompositions. This employs randomized sketching \cite{halko2011finding} ---  similar to what we consider in this work --- but directly for compression as opposed to indirectly through regularization. We highlight the single-pass randomized methods using the SVD \cite{yu2017single} and interpolative decomposition (ID) \cite{dunton2020pass,pacella2022task,li2025online} for their focus on \textit{in situ} compression.
                
            One of the more significant directions of neural compression research uses INR in implicit neural compression (INC). Here, compression is accomplished by simply requiring that the size of the network is smaller than the size of the data so that \cref{eq:cr} is greater than \(1\). Other approaches to compression include autoencoders, with recent developments to enable arbitrarily structured meshes \cite{doherty2024quadconv}. A mixed autoencoder INR approach is also developed in \cite{pham2023autoencoding}. The work of \cite{yang2023tinc} introduces a blocking tree-based method with parameter sharing, and multiscale representations have also been considered in \cite{saragadam2022miner,tang2024ecnr}.
    
            Most INC research has been focused on images \cite{strumpler2022implicit,dupont2021coin} or video \cite{zhang2021implicit, chen2021nerv}, and some for general signals \cite{pistilli2022signal}. More relevant to our task are \cite{lu2021compressive,sales2024implicit}, which apply INC to three and four-dimensional volumetric data from scientific simulations. Importantly, these works consider the offline scenario, incurring large training times and often using the entire dataset for meta-learning in a pre-training stage. Most also incorporate some form of quantization to reduce the model parameters, and thus increase the compression rate further.

            A hybrid approach is considered in \cite{madireddy2021}, where a convolutional neural network is used to correct JPEG-based compression via artifact removal. A pre-trained model is fine-tuned \textit{in situ}, in part using a buffer of stored data, before it is deployed to compress, i.e., remove artifacts, in a full simulation. While their use of experience replay ideas is similar to ours, the fundamental process is significantly different, given our sketch-based approach. Additionally, similar to many offline neural compression methods, their work requires significant processing outside of the \textit{in situ} learning stage.
    
            The most similar work to our own is the \emph{knowledge distillation} approach of \cite{han2023kd} (KD-INR), which also uses INR, but not hypernetworks. The difficulty associated with \textit{in situ} learning is avoided by breaking the problem down into individual offline learning tasks. In particular, a new time-independent INR is trained at every time step, and then these are all distilled into a single time-dependent INR using a final knowledge transfer training stage. This final stage requires an extra training step, relative to our method, after the simulation has concluded. Additionally, this approach uses an offline pre-processing step to normalize the data and does not consider multi-channel training, although it is unclear how this impacts performance. Without the offline normalization step, KD-INR is also an \textit{in situ} method, but it incurs different costs from our own approach, due to saving different objects. We did not compare directly to KD-INR because the code and data were not publicly available, so this is left to future work.

    \section{\textit{In Situ} Training}\label{sec:insitu-training}

        The \textit{in situ} setting is difficult because it limits the amount of training data the network is exposed to, and the sequential nature of the process can cause catastrophic forgetting of previously learned information. We propose extending INC to the \textit{in situ} mode via a simple procedure that may be of independent use, as it does not depend on the network architecture in any particular way. Our procedure essentially falls under the experience replay umbrella as a method to regularize the learning process, but our twist is to do so with sketched data samples. Doing so enables us to store data from more time steps, as the size of that data is much smaller.

        \begin{figure}[hb]
            \centering
            \includegraphics[width=0.5\textwidth]{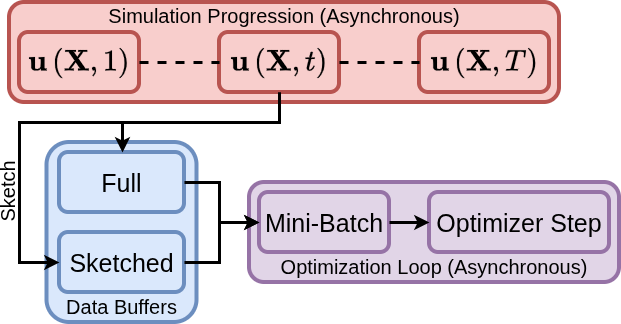}
            \caption{Our \textit{in situ} INC training approach relies on a sketched data buffer to avoid catastrophic forgetting.}
            \label{fig:buffer}
        \end{figure}
        
        To ensure that the data is used most effectively in this single-pass regime, we first assume that we have some amount of available working memory, large enough to hold multiple snapshots at once. This memory is used to implement a buffer which holds data extracted from the simulation, visualized in \cref{fig:buffer}. The first portion of this buffer holds full snapshots in a queue of size \(T_f\ll T\) --- in fact, our numerical experiments use \(T_f=1\). Newly generated snapshots are pushed into the queue, replacing old ones as the maximum size is reached. The second portion of the buffer contains sketched snapshots; the overall size of this buffer is not large, but its length \(T_s\leq T\) can be quite long since the dimension of each sketch is very small. As we shall explain in further in \Cref{sec:sketching}, sketching refers to a linear dimensionality reduction technique wherein a linear operator \(\bm{S}:\reals^{n\times c}\to\reals^{k\times c}\) projects an \(n\)-dimensional vector onto a \(k\)-dimensional subspace, with \(k\ll n\), such that some information in the original data is preserved. The total size of the stored buffer data is then given as follows:
        \[(T_f\times n + T_s\times k)\times c.\]
        The optimization loop then extracts a mini-batch (sub-sample) of full and sketched samples from the buffer to use in its update step. One should note that the simulation and optimization processes may be asynchronous, and in practice, many optimization iterations may take place in between new snapshots being generated.

        To frame our \textit{in situ} loss, it helps to consider the ideal loss we would use in the streaming case if one had no memory constraints. Let \(\mathcal{L}\) be a given loss function, then the ideal loss for a reconstructed snapshot \(\bm{\widetilde{U}}_{t'}= \inr(\bm{X},t';\bm{\theta})\) is given as
        \begin{equation} \label{eq:ideal-loss}
            \mathcal{L}_\text{ideal}(\inr, t) =\frac{1}{t} \sum_{t'=1}^t \mathcal{L}(\bm{U}_{t'}, \bm{\widetilde{U}}_{t'}),
        \end{equation}
        which requires storing \(\{\bm{U}_{t'}\}_{t'=1}^{t}\) for \(t=1,\ldots,T\) and is thus impractical. Instead, we use \cref{eq:insitu-loss}, for full and sketched batch sizes \(b_f\) and \(b_s\), respectively.
        \begin{equation}\label{eq:insitu-loss}
            \mathcal{L}_\text{insitu}(\inr, t) = \underbrace{\frac{1}{b_f}\sum_{m=1}^{b_f}\mathcal{L}(\bm{U}_{t_m}, \bm{\widetilde{U}}_{t_m})}_{\mathcal{L}_{\text{full}}} + \underbrace{\frac{\lambda}{b_s}\sum_{m=1}^{b_s}\mathcal{L}(\bm{S}_{t_m}\bm{U}_{t_m}, \bm{S}_{t_m}\bm{\widetilde{U}}_{t_m})}_{\mathcal{L}_{\text{sketch}}},
        \end{equation}
        where \(\lambda>0\) weights the sketch regularization term. In the case of \(T_f=1\), the first term simplifies significantly. The mini-batch samples are always drawn anew at every iteration of the optimization algorithm, but the sketch \(\bm{S}_{t'}\) is only drawn once per time-step and then reused. It is also important to note that the product \(\bm{S}_{t_m}\bm{U}_{t_m}\) is stored in the sketched buffer, as it cannot be reformed since we may not have saved the full snapshot. The sketched reconstruction \(\bm{S}_{t_m}\bm{\widetilde{U}}_{t_m}\) is formed by sketching the full INR reconstruction, because only subsampling sketches can be applied by evaluating the network at a specific set (the subsample) of points.
        
        \begin{algorithm}[ht]
            \setstretch{1.4}
            \caption{\textit{In Situ} Training with Sketching}\label{alg:insitu-training}
            \KwData{Sketch size \(k\), buffer sizes \(T_f\) \& \(T_s\), batch sizes \(b_f\le T_f\) \& \(b_s\le T_s\)}
            \(\inr \gets \text{build\_model}(~)\) \tcp*{Construct INC neural network}
            \(\bm{B}_{f} \gets \text{queue}(T_f)\) \tcp*{Full queue}
            \(\bm{B}_{s} \gets \text{queue}(T_s)\) \tcp*{Sketched queue}
            \While{simulation is running}{
                \Case{new snapshot \(\bm{U}_t\) is available}{
                    \(\bm{S}_t \gets \text{construct\_sketch}(k)\) \tcp*{Construct sketch (storing seed)}
                    \(\text{push}(\bm{B}_{f}, \bm{U}_t)\) \tcp*{Push full snapshot into queue}
                    \(\text{push}(\bm{B}_{s}, \bm{S}_t\bm{U}_t)\) \tcp*{Push sketched snapshot into queue}
                }
                \For{max training cycles}{
                    \(\bm{U}_f \gets \text{minibatch}(\bm{B}_{f}, b_f)\) \tcp*{Sub-sample batch from full queue}
                    \revb{
                    \(\bm{\widehat{U}}_s \gets \text{minibatch}(\bm{B}_{s}, b_s)\) \tcp*{Sub-sample batch from sketched queue}
                    \(\text{compute\_loss}(\inr;\bm{U}_f,\bm{\widehat{U}}_s)\) \tcp*{Via \cref{eq:insitu-loss} with explicit minibatches}
                    \(\text{update\_model}(\inr)\) \tcp*{Optimizer step to update parameters}
                    }
                }
            }
        \end{algorithm}

        \subsection{Sketching}\label{sec:sketching}

            Sketches may be deterministic or randomized. Often, in the latter case, we use random matrix ensembles based on specific distributions such as the Gaussian. Random subsampling can also be cast in this form by choosing a random subset of the \(n\)-dimensional standard basis. Gaussian ensembles generally perform well and have a plethora of theoretical results, but require slow, dense matrix operations. On the other hand, subsampling is extremely fast but can perform quite poorly in specific scenarios. In the worst-case, subsampling may select points that are the least helpful for the desired task, whereas dense sketches always aggregate information from all points. However, subsampling does preserve part of the mesh structure, as it is coarsening instead of aggregating, which can be beneficial in certain settings. Deterministic sketches with this preservation property have been studied in \cite{dunton2021deterministic}, with the idea that sketching operators informed by the data itself can outperform, both in speed and approximation quality, randomized sketches which are usually data oblivious.

            It is important to make note of two aspects of a practical application of sketching. First, it is likely that the data or signal \(\bm{u}\) has more than one channel, in which case the same sketch/transform is broadcast across them. Second, when the sketch is random, a random seed is generated for each snapshot and saved in the buffer alongside the sketched snapshot itself. When sketching for the loss in \cref{eq:insitu-loss}, this seed is used in the random number generation for that sketch to ensure it is the same as the one used originally for the data. This does incur a minor amount of extra memory overhead due to the \(T_s\) seed integers that must be stored alongside the sketched samples themselves, but this is minimal and avoids storing the entire sketch itself.

            We elect to use the fast Johnson-Lindenstrauss transform (FJLT) \cite{ailon2009fast} due to its superior speed and empirical performance compared to other possible options. However, we will also investigate the performance of subsampling in \cref{sec:experiments} --- and it is also used as a subroutine for the FJLT sketch --- so we present this in \cref{alg:subsample}. The FJLT is a Johnson-Lindenstrauss transform (JLT): a transformation that projects a finite set of points into many fewer dimensions while approximately preserving pairwise relative distances. The JL result \cref{thm:jlt} proves the existence of such a mapping, and any such mapping that satisfies this lemma may be referred to as a JL transform.

            \begin{theorem}[Johnson-Lindenstrauss Transformations \cite{johnson1984extensions}]\label{thm:jlt}
                For any $\epsilon>0$ and $n\in \mathbb{N}$, there exists a probability distribution \(\mathcal{D}\) with support in \(\reals^{k\times n}\), where \(k\leq \min\{n,\OO(\epsilon^{-2}\log(m))\}\), such that for any set of $m$ points \(\{\bm{p}\}_{i=1}^m\subset\reals^n\), there exists $1>\delta>0$ such that if \(\bm{S}\sim \mathcal{D}\) then the following holds with probability at least \(1-\delta\):
                \begin{equation}\label{eq:jlt}
                    (1-\epsilon)\|\bm{p}_i-\bm{p}_j\|_2^2 \leq \|\bm{S}\bm{p}_i-\bm{S}\bm{p}_j\|_2^2 \leq (1+\epsilon)\|\bm{p}_i-\bm{p}_j\|_2^2 \quad \forall i,j\in[m].
                \end{equation}
                In particular, for sufficiently large dimensions, using the probabilistic method, we see that there exists a linear map \(\bm{S}: \reals^n \to\reals^k\) with \(k=\OO(\epsilon^{-2}\log(m))\) such that \cref{eq:jlt} holds.  
            \end{theorem}

            \begin{definition}[JL Transform]\label{def:jlt}
                If \(\mathcal{D}\) is a probability distribution that satisfies the statements of \cref{thm:jlt}, we say \(\mathcal{D}\) is an \((\epsilon,\delta)\)-JL transform (with dimensions and sizes inferred from context). Following conventions in the literature, we often informally refer to a random matrix \(\bm{S}\sim \mathcal{D}\) as a JL transform as well.
            \end{definition}

            Furthermore, the construction of the distribution \(\mathcal{D}\) is explicit, with the JLT first implemented via an orthonormalized Gaussian random matrix, for which \cref{eq:jlt} holds with \(k=\OO(\epsilon^{-2}\log(m))\) \cite{johnson1984extensions}. The practical downside is the slow, dense matrix-vector multiplications. The FJLT, on the other hand, uses a significantly different approach to gain a similarly significant speedup. \Cref{alg:fjlt} describes this process, which uses the orthogonal discrete cosine transform (DCT) to effectively mix rows \cite{avron2010blendenpik}\footnote{Specifically, a DCT of type II or III.}. \Cref{alg:fjlt} satisfies \cref{eq:jlt} with \(k=\OO(\epsilon^{-2}\log(m(\delta-n^{-\log^3(n)})^{-1})\log^4(n))\) (see, e.g., \cite{huang2021spectral}) and, due to the DCT, only incurs an \(\OO(n\log(n))\) time cost since the DCT can be applied using fast Fourier transform techniques. As noted earlier, the random seed for the Rademacher random variable and subsampling are saved so that the same FJLT transform can be reused later.

            \begin{algorithm}
                \setstretch{1.4}
                \caption{Subsampling-Sketch \(\bm{S}_\text{subsample}\bm{U}\)}\label{alg:subsample}
                \KwData{Data \(\bm{U}\in\reals^{n\times c}\), projected dimension \(k\leq n\)}
                \(\Omega \gets \operatorname{random\_choice}(\{1,\ldots,n\},k)\) \tcp*{Select \(k\) indices uniformly without replacement}
                \Return \(\bm{U}_{\Omega,:}\) \tcp*{Sample rows}
            \end{algorithm}
            
            \begin{algorithm}
                \setstretch{1.4}
                \caption{FJLT-Sketch \(\bm{S}_\text{FJLT}\bm{U}\)}\label{alg:fjlt}
                \KwData{Data \(\bm{U}\in\reals^{n\times c}\), projected dimension \(k\leq n\)}
                \(\bm{d} \gets \operatorname{rademacher}(n)\in\{\pm 1\}^n\) \tcp*{Independent, uniformly random signs}
                  \For{$i\gets1$ \KwTo $c$}{
                \(\widehat{\bm{U}}_{:,c} \gets \sqrt{n/k}\operatorname{DCT}(\bm{U}_{:,c}\odot\bm{d})\) \tcp*{\(\odot\) is element-wise/Hadamard multiplication}
                }
                \Return \(\operatorname{Subsampling-Sketch}(\widehat{\bm{U}},k)\) \;
            \end{algorithm}

        \subsection{JL Regularization for \textit{In Situ} Learning}\label{sec:theory}
        
            The goal of including the sketch loss in \cref{eq:insitu-loss} is to prevent forgetting past information from previously encountered samples and thus losing reconstruction accuracy. \Cref{sec:sketching} suggests informally that for a sketch satisfying \cref{thm:jlt}, we might be able to preserve the information embedded in the sketched data. But why should a sketch-based loss effectively serve as a regularization against catastrophic forgetting? In our main theoretical result, \cref{thm:jlt-surrogate}, we will see that a sketch-based loss can serve as a surrogate to the true loss, i.e., that evaluated on the full data, for past samples.
            
            First, we should clarify, for any given sketch, what exactly is being sketched. At first glance, for snapshot \(t\), this is easily identified as the true data \(\bm{U}_t\) and the reconstructed output of the INR on the entire mesh \(\bm{\widetilde{U}}_t=\inr(\bm{X},t)\). However, this ignores the dependence of the INR on the parameters \(\bm{\theta}\), which means we are sketching \(\inr(\bm{X},t;\bm{\theta})\) for any possible \(\bm{\theta}\). This violates \cref{thm:jlt} because we are no longer sketching a finite set of points, given that the parameters are continuous. However, as we see in \cref{thm:manifold-jlt}, one can extend the Johnson-Lindenstrauss result to sketches on manifolds with an infinite number of points. 

            We refer to \cite{Lee_manifoldBook2ndEd} for an extensive treatment of manifolds, but it is sufficient for our purposes to understand that a Riemannian manifold can be viewed as a subset of Euclidean space and behaves locally like a vector space. Informally, we are considering (possibly) lower-dimensional geometric structures embedded in some ambient subset of \(\reals^N\).

            \begin{theorem}[Manifold JL Transform \cite{baraniuk2009random}]\label{thm:manifold-jlt}
                Let \(\mathcal{M}\) be a compact \(M\)-dimensional Riemannian manifold of \(\reals^N\) having volume \(V\), condition number \(\tau^{-1}\), and geodesic covering regularity \(R\). Fix \(\epsilon,\delta\in(0,1)\) and let \(\bm{S}\) be a normalized random ortho-projector\footnote{A random \(k\times N\) matrix with orthogonal rows normalized by \(\sqrt{N/k}\).} from \(\reals^N\) to \(\reals^k\) which satisfies the following:
                \[k = \OO\left(\frac{M\log\left(NVR\tau^{-1}\epsilon^{-1}\right)\log\left(\delta^{-1}\right)}{\epsilon^2}\right).\]
                If \(k\leq N\), then with probability at least \(1-\delta\) the following holds for \(\bm{p},\bm{q}\in\mathcal{M}\):
                \[(1-\epsilon)\|\bm{p}-\bm{q}\|_2 \leq \|\bm{Sp}-\bm{Sq}\|_2 \leq (1+\epsilon)\|\bm{p}-\bm{q}\|_2.\]
            \end{theorem}

            The volume, condition number, and geodesic covering regularity (the latter two of which are defined in \cite{baraniuk2009random}) together define the topological regularity of the manifold. Note that this result, as originally stated, includes a factor of \(\sqrt{k/N}\) as the sketch is not assumed to be normalized. \Cref{thm:manifold-jlt} presents a very similar result to \cref{thm:jlt}, but where the size of the sketch depends on geometric properties of the manifold, most importantly the intrinsic dimension, and not on the collection of points being sketched. This means that for low-dimensional manifolds of a potentially high-dimensional ambient space, one can construct a small JL transform for an arbitrarily large number of points. The size \(k\) depends only logarithmically on the ambient dimension \(N\), and in fact, this dependence can be removed completely as shown in subsequent works \cite{ArminMike2015}. This ability to construct JL transformations for manifolds is key to our main result, which we present next.

            \begin{theorem}[JL Surrogate]\label{thm:jlt-surrogate}
                Define \(\mathcal{L}\) as the squared \(\ell_2\) (i.e., unnormalized MSE) loss. For a fixed mesh \(\bm{X}\) and time \(t\in[T]\), with parameters \(\bm{\theta}\in\reals^N\), suppose the map \(\inr:\reals^N\to\mathcal{M}\subset\reals^{n\times c}\), defined via the neural network as \(\inr(\bm{\theta}):=\inr(\bm{X},t;\bm{\theta})\), maps to the Riemannian manifold \(\mathcal{M}\). Further, suppose the sketch \(\bm{S}_t:\reals^{n\times c}\to\reals^{k\times c}\) is a \((\epsilon,\delta)\)-JL transform on the shifted manifold \(\mathcal{M}-\{\bm{U}_t\}\), then, \(\forall \bm{\theta}\in\reals^N\), the following holds with probability at least \(1-\delta\):
                \[\mathcal{L}(\bm{U}_t,\inr(\bm{\theta})) \leq \frac{1}{1-\epsilon}\mathcal{L}(\bm{S}\bm{U}_t,\bm{S}\inr(\bm{\theta}));\]
                and, also with probability at least \(1-\delta\),
                \[\mathcal{L}(\bm{U}_t,\inr(\bm{\theta}_S^\star)) \leq \frac{1+\epsilon}{1-\epsilon} \min_{\bm{\theta}} \mathcal{L}(\bm{U}_t,\inr(\bm{\theta})),\]
                where \(\bm{\theta}_S^\star\) is a minimizer of the sketched loss.
            \end{theorem}
            
            \begin{proof}                
                The unnormalized MSE loss is given as \(\mathcal{L}(\bm{p},\bm{q})=\|\bm{p}-\bm{q}\|_2^2\). Via the left inequality of \cref{thm:manifold-jlt},
                \begin{equation*}
                    \mathcal{L}(\bm{U}_t,\inr(\bm{\theta})) = \|\bm{U}_t-\inr(\bm{\theta})\|_2^2 \leq \frac{1}{1-\epsilon} \|\bm{S}\left(\bm{U}_t-\inr(\bm{\theta})\right)\|_2^2 = \frac{1}{1-\epsilon} \mathcal{L}(\bm{S_t}\bm{U}_t,\bm{S}_t\inr(\bm{\theta})),
                \end{equation*}
                which is the first desired inequality. Next, define \(\bm{\theta}^\star\) as a minimizer of the full loss \(\bm{\theta}\mapsto \mathcal{L}(\bm{U}_t,\inr(\bm{\theta}))\) and \(\bm{\theta}_S^\star\) as a minimizer of the sketched loss \(\bm{\theta}\mapsto\mathcal{L}(\bm{S}\bm{U}_t,\bm{S}\inr(\bm{\theta}))\). Then, we have the following:
                \begin{align*}
                    \mathcal{L}(\bm{U}_t,\inr(\bm{\theta}_S^\star)) &\leq \frac{1}{1-\epsilon}\mathcal{L}(\bm{S}\bm{U}_t,\bm{S}\inr(\bm{\theta}_S^\star)) &\text{via \cref{thm:jlt-surrogate}}\\
                    &\leq \frac{1}{1-\epsilon}\mathcal{L}(\bm{S}\bm{U}_t,\bm{S}\inr(\bm{\theta}^\star)) & \text{since \(\bm{\theta}_S^\star\) is a minimizer} \\
                    &\leq \frac{1+\epsilon}{1-\epsilon}\mathcal{L}(\bm{U}_t,\inr(\bm{\theta}^\star)), &\text{via right inequality of \cref{thm:manifold-jlt}}
                \end{align*}
                which is the second desired inequality.
            \end{proof}

            The proof of the second inequality is motivated by the informal proof of Theorem 2.14 in \cite{woodruff2014sketching}. Also, note that we construct a JL transform for the shifted manifold \(\mathcal{M}-\{\bm{U}_t\}\), which is just a convenience to simplify the proof, but does not introduce any complexity to the sketch itself. The shifted manifold will have the same geometric properties as \(\mathcal{M}\) that impact the construction of the sketch.

            The first bound in \cref{thm:jlt-surrogate} says that if we can find \(\bm{\theta}\) to make the sketched loss small, then the true loss is also small, or, from another perspective, that for any given snapshot, the sketch loss serves as an approximate surrogate for the true loss. One may observe that the tightness of this bound is controlled by \(0<\epsilon<1\), which also impacts the sketch size from \cref{thm:manifold-jlt}. For our goal of regularization, we may not need \(\epsilon\) to be that small, thus allowing us to use smaller sketches. However, we don't view this as a full explanation of our performance in practice. The trade-off that ideally holds, and which is empirically validated, is that much smaller sketches can still serve well enough as an approximate surrogate to the true loss to be an effective regularizer against forgetting. The actual learning is then performed based on the full snapshots.
            
            The second bound says that we can make the minimizer of the sketched loss nearly as good as the minimizer of the true loss. We will make use of this bound in \cref{sec:experiments} to loosely compare our chosen sketch sizes to what is suggested by the theory.

            \Cref{thm:jlt-surrogate} holds for a single snapshot, but in practice, we are sketching multiple snapshots at once to compute the overall sketch loss. The number of sketches actually depends on how long the process has been running and the size \(T_\text{sketch}\) of the sketch buffer.

            \begin{theorem}[Batch JL Surrogate]\label{thm:batch-jlt-surrogate}
                For a sketch buffer with size \(T_s\), \cref{thm:jlt-surrogate} holds for the batch loss, i.e., the loss averaged over multiple snapshots, when the sketches are constructed with failure probability \(\delta/T_s\).
            \end{theorem}

            \Cref{thm:batch-jlt-surrogate} extends the result of \cref{thm:jlt-surrogate} to multiple snapshots, which follows from a simple union bound over the individual sketches. To maintain the same probability of success, the probability of failure \(\delta\) for each sketch must be decreased. This can be done at minimal cost due to the logarithmic dependence on this parameter in \cref{thm:manifold-jlt}.

            We make a key assumption in \cref{thm:jlt-surrogate} that \(\inr(\bm{\theta})\) maps to a Riemannian manifold, but is this actually true? Moreover, even if this assumption does hold, are the other relevant geometric properties from \cref{thm:manifold-jlt} --- dimension, volume, condition number, and geodesic covering regularity --- in a way such as to yield a small sketch size? Rigorously answering this question is difficult. Considering just the dimension, we expect this to be dependent on both the data and the network parameters. For the former, in the worst case scenario, the data is entirely uncorrelated, so the dimension of \(\mathcal{M}\) may be as large as \(n\times c\). Similarly, in the best-case scenario, the data is perfectly correlated, so the dimension of \(\mathcal{M}\) could be much smaller. In lieu of a theoretical answer to whether the INR output lies on a low-dimensional manifold, we provide computational evidence in \cref{sec:results} that it indeed does.

    \section{Experiments}\label{sec:experiments}

        We conduct our experiments on the datasets listed in \Cref{tab:datasets}, which highlight several potentially challenging or important attributes. \reva{The Ignition dataset\footnote{\reva{An internal dataset that was also used in \cite{doherty2024quadconv}.}}} describes a fully resolved gas wave front on a 2D uniform grid that transitions from a transport phase to a steady state jet. The Neuron dataset describes a 3D diffusion process in an unstructured, non-Cartesian branching neuron tree~\cite{Angran3D}. The Channel dataset, from the Johns Hopkins Turbulence Database \cite{channel_flow_1, channel_flow_2}, describes 3D turbulent flow on a non-uniform grid. It has been trimmed in space to a \(64^3\) volume and in time to the first 500 snapshots.  We define the notion of a ``sample factor'' as the relative size of the sketch to the mesh size for a given dataset, i.e., \(100\cdot k/n\).

        \begin{table}[ht]
            \centering
            \resizebox{\textwidth}{!}{
            \begin{tabular}{llS[table-format=1.0]S[table-format=3.0]S[table-format=6.0]S[table-format=1.0]S[table-format=4.0,round-mode=places,round-precision=0]}
                \toprule
                 & \textbf{Mesh Type} & \textbf{Spatial Dimension} & {\textbf{Snapshots}} & {\textbf{Mesh Nodes}} & \textbf{Channels} & {\textbf{Memory} (MB)}\\ 
                \textbf{Datasets} & & {\(d\)} & {\(T\)} & {\(n\)} & {\(c\)} & \\ 
                \midrule
                \textbf{Ignition} & uniform & 2 & 450 & 2500 & 4 & 18 \\ 
                \textbf{Neuron} & non-uniform & 3 & 500 & 116943 & 1 & 233.886 \\
                \textbf{Channel} & non-uniform & 3 & 500 & 262144 & 3 & 1572.864 \\
                \bottomrule
            \end{tabular}
            }
            \caption{Dataset descriptions. Total memory is calculated as \num{4e-6} MB per 32-bit floating point value.}
            \label{tab:datasets}
        \end{table}

        \Cref{eq:loss} provides the \revb{snapshot} relative loss function we use for training, \cref{eq:relative-error} defines the full dataset relative Frobenius error, and \cref{eq:psnr} defines the \revb{average snapshot} peak signal to noise ratio (PSNR). The latter two of these three serve as our primary performance metrics.
        \begin{gather}
            \mathcal{L}(\bm{U}_t,\bm{\widetilde{U}}_t) = \frac{1}{C}\sum_{c=1}^C\frac{\|(\bm{U}_t)_{:,c}-\bm{(\widetilde{U}}_t)_{:,c}\|_2}{\|(\bm{U}_t)_{:,c}\|_2} \label{eq:loss}, \\
            \text{RFE}(\bm{U},\bm{\widetilde{U}}) = \frac{1}{C}\sum_{c=1}^C\frac{\|\bm{U}_{:,:,c}-\bm{\widetilde{U}}_{:,:,c}\|_F}{\|\bm{U}_{:,:,c}\|_F} \label{eq:relative-error}, \\
            \revb{\text{PSNR}(\bm{U},\bm{\widetilde{U}}) = \frac{1}{T}\sum_{t=1}^{T}\frac{1}{C}\sum_{c=1}^C 20\cdot\log_{10}\left(\frac{\max_i(\bm{\widetilde{U}}_t)_{i,c}}{\|(\bm{U}_t)_{:,c}-(\bm{\widetilde{U}}_t)_{:,c}\|_2}\right).} \label{eq:psnr}
        \end{gather}
        
        Our experiments are designed to achieve two primary goals. First, to show that we can achieve strong performance, both offline and \textit{in situ}, using our INC methods. Second, and more importantly, to show that we can approximately match the offline results using our \textit{in situ} training paradigm. We do not optimize our models for maximum performance, nor do we further compress the data using other techniques such as network quantization. Further effort will likely yield improved results, but they are not the focus of this paper and are therefore left for future work. Instead, we choose simple models that yield desirable compression rates and relative error of a few percent.
        
        Our work is implemented in PyTorch \cite{pytorch} and available as an open-source GitHub repository \cite{paper-repo}. All training is conducted on a single GPU (A100 or L40S) using the RAdam optimizer, a learning rate of \(10^{-4}\), and in single precision (i.e., 32-bit floating point) \cite{single_precision} for both the parameters and data. Given that our goal is to compress the given data, we do not seek to generalize to unseen data; thus, overfitting is not an issue. Our method is fully \textit{in situ}, so we perform no pre-processing of the data. For all \textit{in situ} results, the size of the full buffer is always \(T_f=1\) (only one full snapshot available at a time), and the size of the sketch buffer is \(T_s=T-1\) (one less than the full time horizon). We do not require sketching the final snapshot because the training loop ends after the last full snapshot has been encoded. The batch size of the sketch buffer and the number of epochs per full snapshot vary between datasets, but are in the range 25--45 and 300--500, respectively. \revb{The sketch regularization weight from \cref{eq:insitu-loss} was manually tuned to \(\lambda=5\) and fixed for all experiments.}

        \subsection{Model Architecture}\label{sec:architecture}

            As alluded to previously, we elect to use a hypernetwork approach for constructing our neural compressor. Practically, this results in two networks working together to map space-time coordinates to a reconstruction of the PDE solution field. \Cref{fig:network} displays the generic INR architecture we employ for both the hypernetwork and the target network. \revb{The ``Hypernet'' maps from the time coordinate \(t\) to parameters and the ``Target INR'' ingests these parameters to define its layers which map from the space-time coordinates \((\bm{x},t)\) to the reconstruction \(\widetilde{\bm{u}}(\bm{x},t)\).}
            
            \begin{figure}[ht!]
                \centering
                \begin{subfigure}[b]{0.48\textwidth}
                    \includegraphics[width=\textwidth]{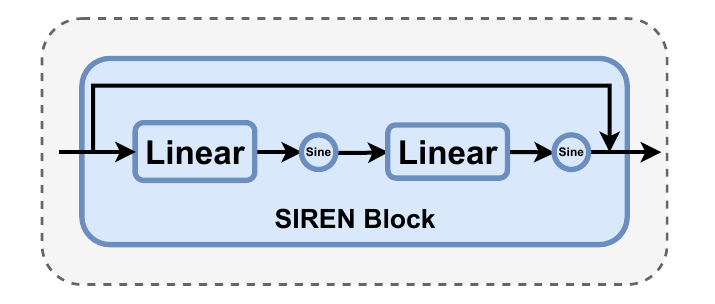}
                    \caption{Network building-block consisting of two linear layers with sine activation functions and a skip connection.}
                    \label{fig:siren-block}
                \end{subfigure}
                \hfill
                \begin{subfigure}[b]{0.48\textwidth}
                    \includegraphics[width=\textwidth]{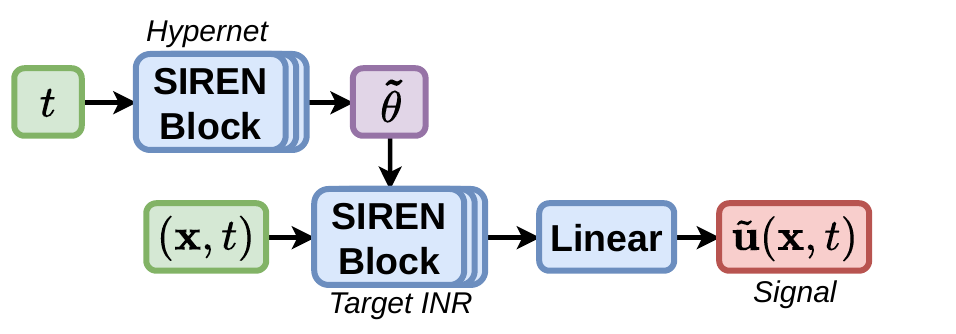}
                    \caption{Overall network structure consisting of the hypernetwork and the target INR.}
                    \label{fig:full-network}
                \end{subfigure}
                \caption{INC network structure.}
                \label{fig:network}
            \end{figure}
            
            Within the sine-layer blocks, we use sinusoidal activation functions with the initialization developed in \cite{sitzmann2020implicit}. We also employ skip connections similar to those of \cite{pan2023neural,sales2024implicit}. Thus, our networks are defined by two hyperparameters: the depth, i.e., the number of blocks, and the width of the internal linear layers. Because we are targeting the \textit{in situ} setting, we cannot effectively normalize the data to ensure it is in some specific range, so no activation is used at the final linear layer. This ensures that the target network can produce any signal value. The lack of activation is also useful for the hypernetwork, as we construct the initial output of the hypernetwork to match the initialization of the target INR. To accomplish this, we use the strategy of \cite{han2023kd}, wherein the bias of the final linear layer in the hypernetwork is set to the exact values given by the initialization of the target INR, and the weights of the last linear layer are scaled by a relatively small factor.
    
            We should note that many more tricks and techniques exist in the literature for both hypernetworks and implicit neural representation, much of which was discussed in \cref{sec:preliminaries}. Our architecture may benefit from these alterations. As an example, a more sophisticated hypernetwork architecture that does not map to all target parameters at once would be far more size-efficient, thus increasing our overall compression rate. However, this is left to future work.

        \subsection{Results}\label{sec:results}
        
            To put our proposed method, coupling \Cref{alg:insitu-training} (buffer) with either \Cref{alg:subsample} (subsampling) or \Cref{alg:fjlt} (FJLT), henceforth called ``InSitu-FJLT'' or ``InSitu-Subsample'' respectively, into context, we compare it with some baseline methods, both offline and \textit{in situ}. We note that all the models use the same overall architecture but differ in how they are trained. Within each dataset, all models have the same number and size of SIREN blocks, and hence have the same compression rate. For offline methods, we first compare with a pure INC model that has access to all the data, using the loss \(\mathcal{L}(\inr,T)\) from \cref{eq:ideal-loss}, and using mini-batches in time for training. This sketching-free method, which we label as ``Offline-Baseline'', represents the best accuracy we could hope for using the INC framework. We also compare with ``Offline-Subsample'' and ``Offline-FJLT'', which have access to all time points \(t\in[T]\) but the spatial dimensions of \(\bm{U}_t\) are sketched, either via \Cref{alg:subsample} or \Cref{alg:fjlt}, respectively. The subsampling sketch has no theoretical guarantees, but for the FJLT sketch, if the manifold hypotheses are valid and we choose sketches of size \(k\) to give an \(\epsilon\)-JLT embedding, then we would expect these errors to be roughly a factor of \(\frac{1+\epsilon}{1-\epsilon}\) worse than the full offline baseline. The ``InSitu-Baseline'' refers to using no sketched buffer at all (i.e.,  \(T_s=0\)), hence it is prone to catastrophic forgetting.
            
            \begin{table}[htbp]
                \centering
                \begin{subtable}[]{\textwidth}
                    \centering
                    \resizebox{\textwidth}{!}{
                    \begin{tabular}{lS[table-format=2.2(2.1)]S[table-format=2.2(2.1)]S[table-format=2.2(2.1)]}
                        \toprule
                        \textbf{Method} & \textbf{Ignition (142\(\times\))} & \textbf{Neuron (1882\(\times\))} & \textbf{Channel (682\(\times\))} \\
                        \midrule
                        Offline-Baseline (full dataset) & 2.17(0.19) & 0.57(0.05) & 5.37(0.36) \\
                        Offline-Subsample (sketched dataset, \Cref{alg:subsample}) & 4.04(0.35) & 14.1(20.4) & 7.67(0.46) \\
                        Offline-FJLT (sketched dataset, \Cref{alg:fjlt}) & 3.87(0.49) & 2.88(2.28) & 7.30(0.26) \\
                        \midrule
                        InSitu-Baseline (no sketched buffer) & 84.2(7.2) & 125.0(23.3) & 77.2(1.6) \\
                        InSitu-Subsample, \Cref{alg:insitu-training} + \Cref{alg:subsample} & 3.89(2.11) & 21.0(36.3) & 5.27(0.46) \\
                        InSitu-FJLT, \Cref{alg:insitu-training} + \Cref{alg:fjlt} & 2.64(0.85) & 0.75(0.10) & 5.22(0.37) \\
                        \bottomrule
                    \end{tabular}
                    }
                    \caption{RFE \(\downarrow\) (\%)}
                \end{subtable}
                \vfill
                \vspace{\baselineskip}
                \begin{subtable}[]{\textwidth}
                    \centering
                    \resizebox{\textwidth}{!}{
                    \begin{tabular}{lS[table-format=2.2(2.2)]S[table-format=2.2(2.2)]S[table-format=2.2(2.2)]}
                        \toprule
                        \textbf{Method} & \textbf{Ignition (142\(\times\))} & \textbf{Neuron (1882\(\times\))} & \textbf{Channel (682\(\times\))} \\
                        \midrule
                        Offline-Baseline (full dataset) & 41.6(0.6) & 60.4(0.4) & 37.4(0.4) \\
                        Offline-Subsample (sketched dataset, \Cref{alg:subsample}) & 37.3(1.1) & 35.4(5.6) & 34.3(0.5) \\
                        Offline-FJLT (sketched dataset, \Cref{alg:fjlt}) & 37.1(0.9) & 48.8(6.4) & 34.6(0.2) \\
                        \midrule
                        InSitu-Baseline (no sketched buffer) & 9.28(1.84) & 14.5(1.6) & 14.3(0.44) \\
                        InSitu-Subsample, \Cref{alg:insitu-training} + \Cref{alg:subsample} & 40.2(1.6) & 40.7(11.8) & 37.4(0.6) \\
                        InSitu-FJLT, \Cref{alg:insitu-training} + \Cref{alg:fjlt} & 41.7(1.2) & 58.2(0.82) & 37.4(0.5) \\
                        \bottomrule
                    \end{tabular}
                    }
                    \caption{PSNR \(\uparrow\) (dB)}
                \end{subtable}
                \caption{Performance metrics on test datasets Ignition, Neuron, and Channel at 5\%, 1\%, and 2\% sketch sample factors (where relevant) --- \reva{which equates to equivalent offline storage costs for 22.5, 5, and 10 full snapshots} --- respectively. Metric statistics are reported as mean plus-or-minus standard deviation from five identical trials. The numbers in the parentheses of the first row of each table are compression ratios.}
                \label{tab:performance}
            \end{table}
            
            \Cref{tab:performance} presents PSNR and RFE performance results on all three datasets in the offline and \textit{in situ} settings. The key observation to be made is that sketching-based regularization often yields offline levels of performance in the \textit{in situ} regime. In particular, sketching with the FJLT can yield results that approximately match those obtained offline across all datasets. While subsample sketching performs equally well on the Channel and almost as well on the Ignition datasets, its performance on the Neuron dataset is striking. Although it is not obvious from the presented statistics, the subsample results on the Neuron dataset are partially skewed by a few catastrophic failures. However, even in successful runs, the results are not at FJLT levels. Overall, we believe this result is not too surprising, as the non-Cartesian geometry of the neuron tree makes subsampling a particularly poor choice for sketching.

            The results in \cref{tab:performance} also show that our \textit{in situ} methods, which combine training of the full snapshot with regularization from previous sketched snapshots, outperform the Offline-FJLT strategy (which just uses the sketched snapshots). The Offline-FJLT is conceptually simple to implement, since it does the sketching \textit{in situ} and then trains offline. However, since it cannot train on the unsketched data, it incurs the \(\frac{1+\epsilon}{1-\epsilon}\) error. \reva{Our \textit{in situ} methods use sketched data but do not seem to incur this extra error, since the sketching is only intended to prevent forgetting, but the network can still train on unsketched data.} This is also unsurprising, since for human memory, the task of learning something new is usually more difficult than the task of not forgetting something already learned.

            \reva{Reported separately in \cref{tab:zfp}, we have also compared to another \textit{in situ} baseline in the form of ZFP \cite{lindstrom2014zfp}. Here, each snapshot and channel is compressed separately through the \texttt{zfpy} software \cite{zfp-software}. Comparable compression rates to what we obtain with INC are not achievable, but since ZFP is an error-bounded scheme, we can ask for similar RFE errors. At this similar quantitative performance, in terms of our RFE and PSNR metrics, ZFP yields two orders of magnitude lower compression rates.}

            \begin{table}[htbp]
                \centering
                \begin{tabular}{llcS[table-format=2.2]S[table-format=2.2]}
                    \toprule
                    \textbf{Dataset} & \textbf{Method} & \textbf{Compression Rate} & \textbf{RFE (\%)} $\downarrow$ & \textbf{PSNR (dB)} $\uparrow$ \\
                    \midrule
                    \multirow{2}{*}{Ignition}
                    & ZFP (Baseline) & 9.03$\times$ & 3.04 & 40.4 \\
                    & InSitu-FJLT (Proposed) & 142$\times$ & 2.64 & 41.7 \\
                    \addlinespace
                    \multirow{2}{*}{Neuron}
                    & ZFP (Baseline) & 5.42$\times$ & 0.64 & 58.5 \\
                    & InSitu-FJLT (Proposed) & 1882$\times$ & 0.75 & 58.2 \\
                    \addlinespace
                    \multirow{2}{*}{Channel}
                    & ZFP (Baseline) & 5.28$\times$ & 7.13 & 37.7 \\
                    & InSitu-FJLT (Proposed) & 862$\times$ & 5.22 & 37.4 \\
                    \bottomrule
                \end{tabular}
                \caption{\reva{Baseline \textit{in situ} performance of ZFP on test datasets. InSitu-FJLT results have been reproduced from \cref{tab:performance} for comparison.}}
                \label{tab:zfp}
            \end{table}

            In \Cref{tab:manifold-dimension-estimates}, we explore how well the theory from \cref{sec:theory} aligns with our empirical results. Although unreported in \Cref{tab:performance}, we use Offline-FJLT and Offline-Baseline results for \cref{eq:loss} to compute the \(\frac{1+\epsilon}{1-\epsilon}\) ratio from \cref{thm:jlt-surrogate}, which we subsequently use to solve for \(\epsilon\). We also compute principal component analysis (lPCA \cite{fukunaga1971algorithm}) based estimates of the local manifold dimension for the INC networks used with each of the target datasets. This local dimensionality estimate yields a rough idea of the overall manifold dimension. We perform this analysis using the \texttt{scikit-dimension} package \cite{scikit-dimension} by taking perturbations of size \(10^{-5}\) around a nominal point of network parameters. This nominal state is obtained within the \textit{in situ} setting after training on the first snapshot, i.e., the first time sketching is used.

            \begin{table}[ht]
                \centering
                \begin{tabular}{lS[table-format=1.3,round-mode=places,round-precision=3]S[table-format=1.3,round-mode=places,round-precision=3]S[table-format=2.0,round-mode=places,round-precision=0]S[table-format=1.2,round-mode=places,round-precision=2]}
                    \toprule            
                    \textbf{Datasets} & \textbf{Full Loss} & \textbf{Sketch Loss} & \textbf{Est. Manifold Dim. \(M\)} & \textbf{Est. Sample Factor (\%)} \\
                    \midrule
                    Ignition & 0.0262 & 0.0485 & 11 & 1.46 \\ 
                    Neuron & 0.0076 & 0.0326 & 29.454 & 0.03 \\
                    Channel & 0.0519 & 0.0717 & 50.008 & 0.19 \\
                    \bottomrule
                \end{tabular}
                \caption{Estimated sample factors using \cref{thm:manifold-jlt} and \cref{thm:jlt-surrogate}, with lPCA estimates of the local manifold dimension, for INC networks on the target datasets.}
                \label{tab:manifold-dimension-estimates}
            \end{table}

            Combining the observed \(\epsilon\) with the manifold dimension estimates, and ignoring log factors in \cref{thm:manifold-jlt} (i.e., setting $k=M/\epsilon^2$), we compute the estimated sample factors. This is intended to be a rough estimate, but even with these approximations, it suggests that the Ignition dataset needs the greatest sketching dimension, followed by the Channel dataset, and finally the Neuron dataset. Our actual experiments used a $5\%$, $1\%$, and $2\%$ sample factor for the Ignition, Neuron, and Channel datasets, respectively, which matches the ordering suggested by the rough theoretical estimate. This suggests that the notion of manifold dimension of a dataset may indeed be a useful concept.
            
            We can see visualizations of the results in \cref{tab:performance} comparing the original data to the offline baseline and \textit{in situ} FJLT reconstructions for the Ignition and Channel datasets in \cref{fig:ignition-comparison,fig:channel-flow-comparison}, respectively. For the Ignition dataset, \cref{fig:ignition-comparison}, we choose a snapshot from early in the simulation where the jet has not yet reached steady state. Note that the Ignition wavefront is accurately reconstructed by all methods, though the Offline Subsample method does have some small artifacts near the center of the wavefront, and the Offline FJLT has artifacts on the sides of the main jet. As observed in \cref{tab:performance}, the \textit{in situ} methods perform excellently, on par with the Offline-Baseline.
            
            \begin{figure}[ht]
                \centering
                \includegraphics[width=\linewidth]{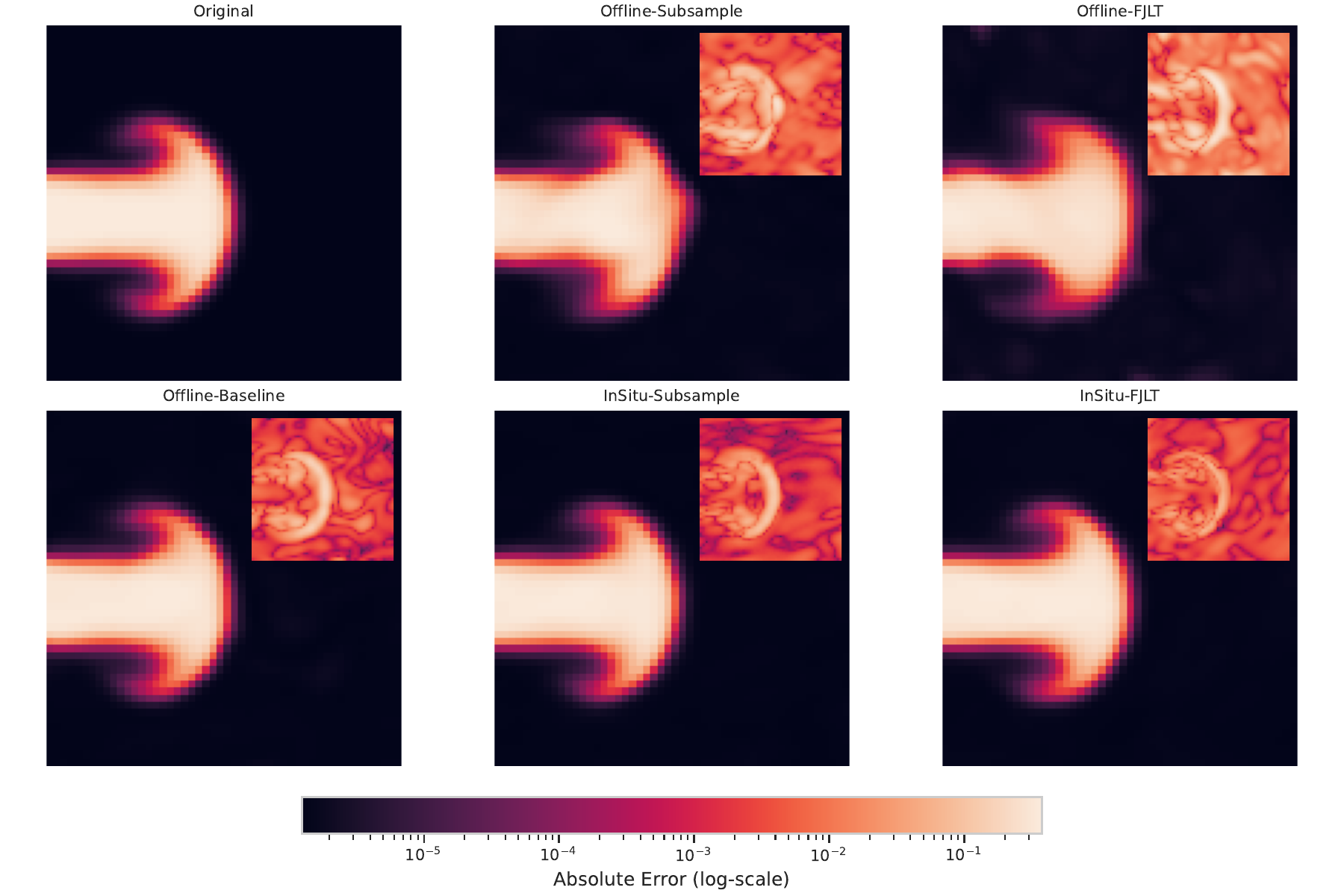}
                \caption{Reconstruction comparison of the Ignition data (channel \(c=2\)) at snapshot \(t=13\). The absolute error in the cutout is computed with respect to the original image, which is presented for reference. The colorbar only references the error cutouts.}
                \label{fig:ignition-comparison}
            \end{figure}
            
            For the Channel dataset, \cref{fig:channel-flow-comparison}, we select an arbitrary snapshot and \(z\)-coordinate for visualization, because the results do not differ across these dimensions due to the nature of the simulation. Across the three feature channels, the \(x\)-velocity (streamwise velocity) appears visually the worst. From a relative error perspective (RFE, \cref{eq:relative-error}), however, this channel performs the best, at about \(1\)\%. The error of the other velocity channels, at around \(6-8\)\%, is what drags down the overall performance. We can attribute this paradox to the difference in scales between the velocity components and artifacts of the relative error metric. The streamwise velocity has a significantly larger scale than the other two components, so despite it appearing worse, it performs better relative to its scale.
            
            \begin{figure}[ht]
                \centering
                \includegraphics[width=\linewidth]{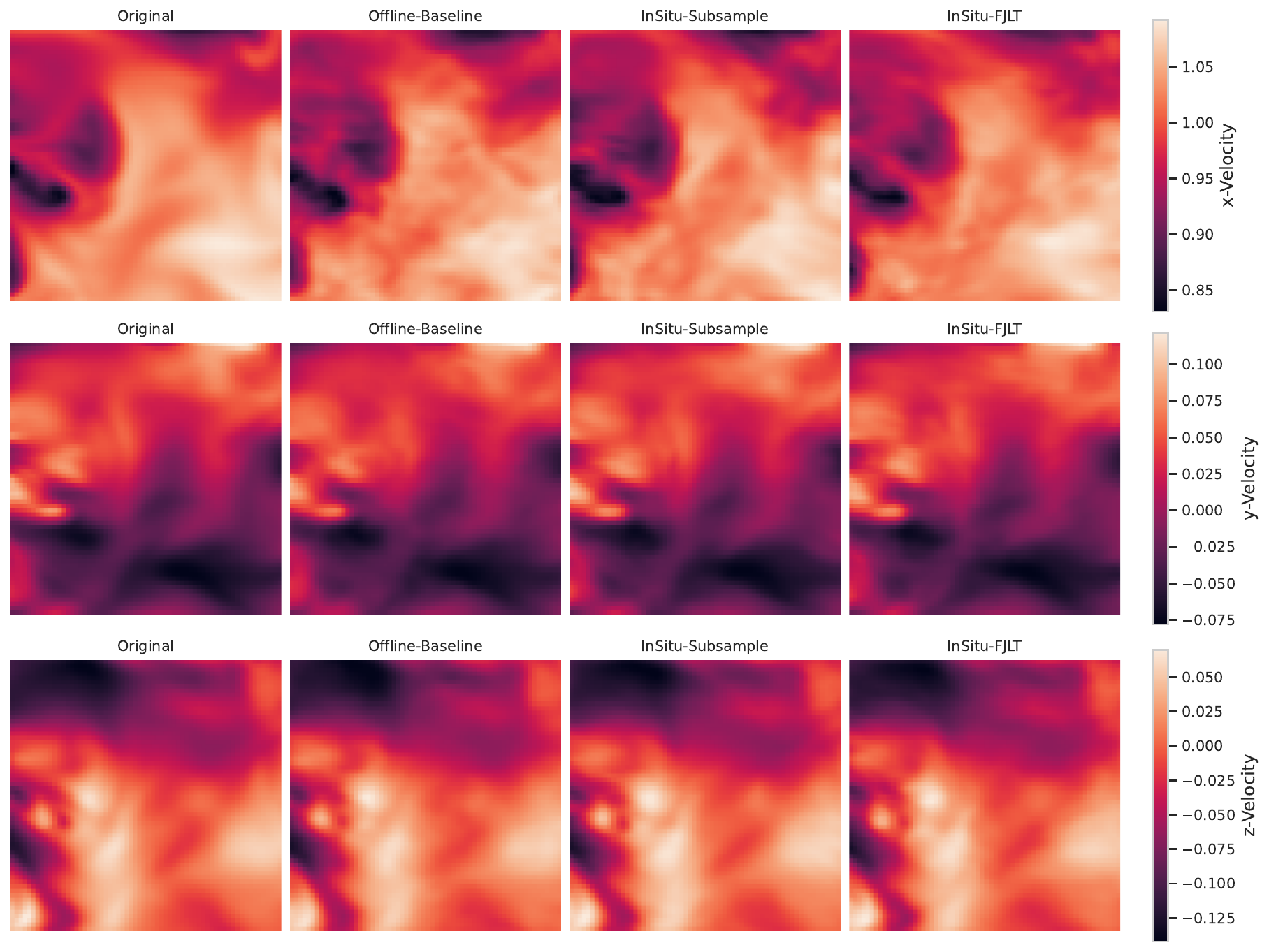}
                \caption{Reconstruction comparison of the Channel dataset on all three channel flow velocities \((x,y,z)\) components in top to bottom order) at snapshot \(t=250\) for a constant \(z\) coordinate in the center of the volume.}
                \label{fig:channel-flow-comparison}
            \end{figure}

            Per-snapshot relative error (\cref{eq:loss}) curves for training and testing on the Ignition dataset are given in \cref{fig:ignition-r3error}. \revb{The training error is split according to \cref{eq:insitu-loss} into the two components \(\mathcal{L}_\text{full}\) and \(\mathcal{L}_\text{sketch}\) corresponding to the two buffers. Thus, the ``Train Error (Full)'' can be viewed as the instantaneous error at the current snapshot, while the ``Train Error (Sketch)'' is the averaged error over previous sketched snapshots. The ``Test Error (Full)'' is simply the instantaneous error for a fixed snapshot after the network has been trained.} For early snapshots, before approximately \(t=50\), one can observe that the testing error is significantly better than the training loss, indicating the sketched snapshots have also actively been used for learning instead of only preventing forgetting. \reva{While this is not necessarily problematic, it may suggest we could alter our training strategy, e.g., by using smaller sketches, to yield a more optimal process.}
            
            \begin{figure}[ht]
                \centering
                \includegraphics[width=\linewidth]{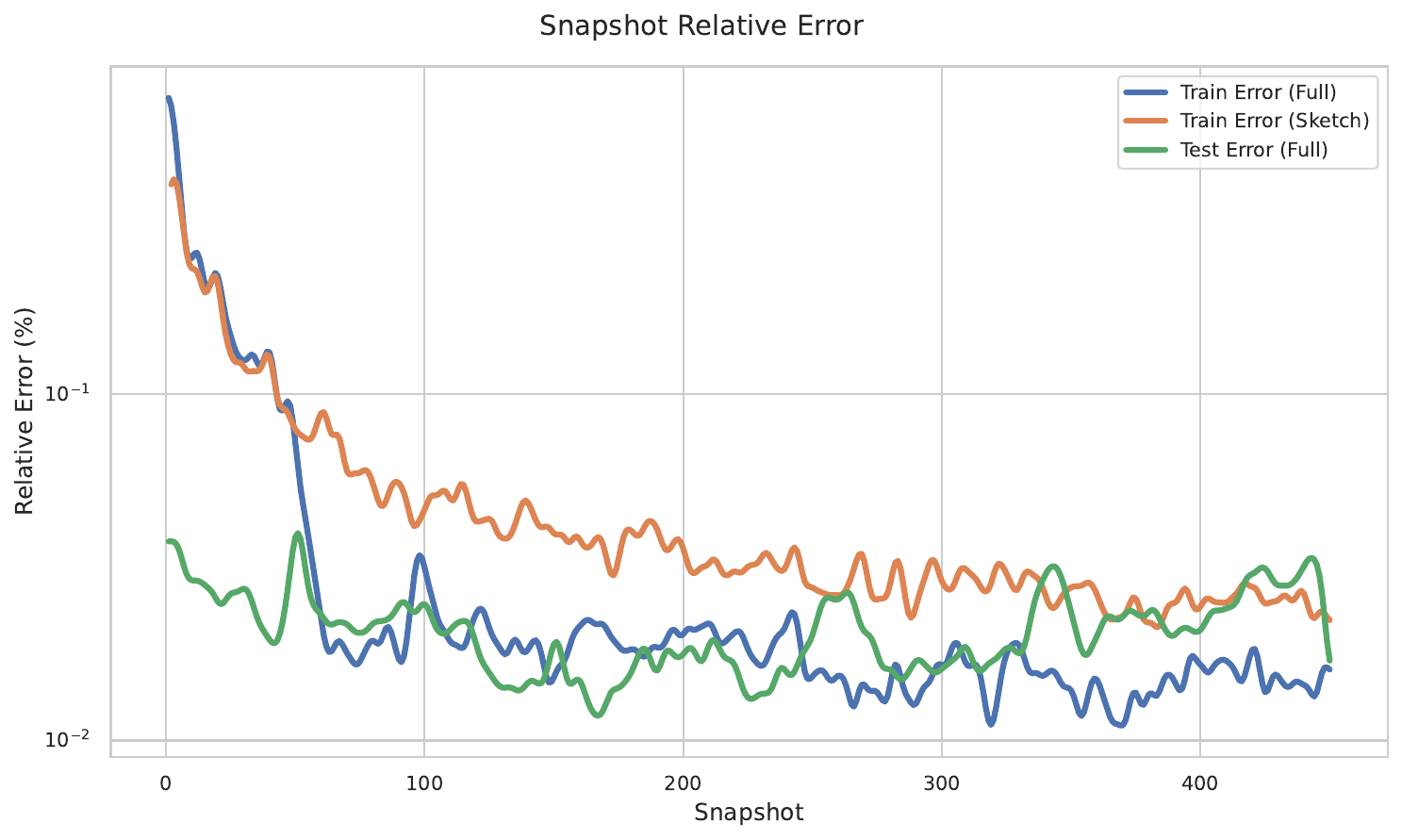}
                \caption{\revb{Training and testing error \cref{eq:loss} on the Ignition data for the ``InSitu-FJLT'' case. Training error is taken at the end of the optimizer steps for a given snapshot and split according to \cref{eq:insitu-loss}. Data has been smoothed via convolution with a Gaussian.}}
                \label{fig:ignition-r3error}
             \end{figure}
            
            \Cref{fig:perf-v-sketch} explores the impact of sketch size on reconstruction performance via five linearly spaced sample factors in \([0.001,0.01]\) and \([0.0055,0145]\) for FJLT and subsample sketching, respectively, on the Neuron dataset. In both RFE and PSNR metrics, FJLT performance improves consistently alongside the sample factor until reaching approximate offline values. As for subsampling, performance also increases with the sample factor, but we observe significantly more instability. In our experience, for larger sample factors, FJLT and subsample sketching lead to similar performance, but one may be ultimately limited in their choice by the amount of offline storage.
            
            \begin{figure}[ht]
                \centering
                \includegraphics[width=\linewidth]{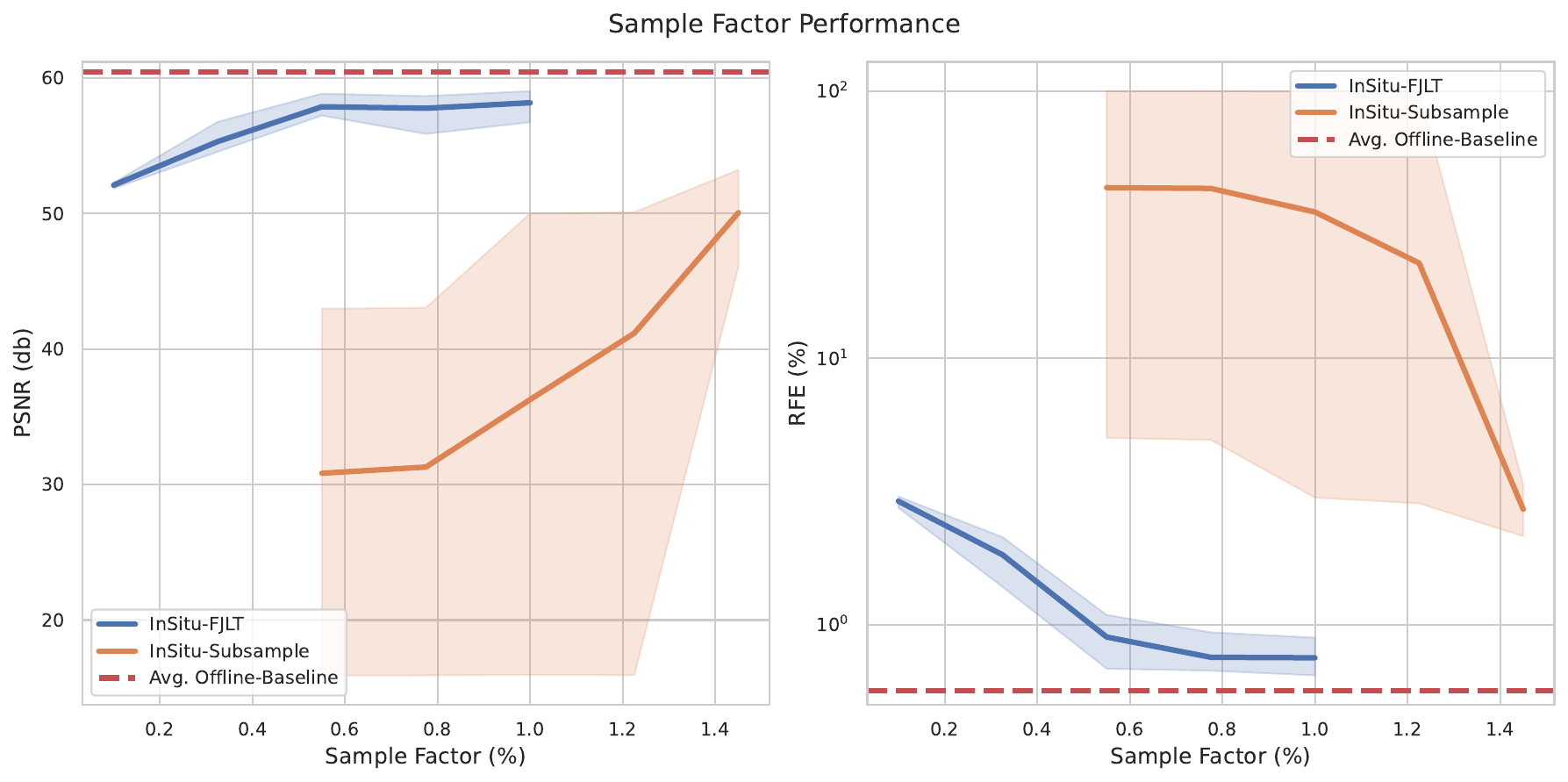}
                \caption{\textit{In situ} Sketch sample factor performance comparison at fixed compression rate on Neuron dataset.}
                \label{fig:perf-v-sketch}
            \end{figure}

            One may ask whether the inclusion of the hypernetwork in our overall INC architecture, shown in \cref{fig:network}, is necessary? Indeed, sketching as a regularization is just as applicable to a time-dependent INR, which we have found to be just as expressive in the offline setting (see \cref{fig:inr-offline}). To probe this question, \Cref{tab:inr-comparison} presents comparison results on the Ignition dataset in the \textit{in situ} regime with and without a hypernetwork. In the latter case, the hypernetwork branch seen in \cref{fig:full-network} is simply eliminated. \revb{Removing the hypernetwork results in fewer total parameters, so the ``INR'' models are appropriately upsized via an additional SIREN block and larger linear layer widths to ensure approximately the same compression rate.} We include results for both our subsample and FJLT sketching options. We observe that the addition of a hypernetwork significantly increases performance regardless of the sketch.
            
            \begin{table}[ht]
                \centering
                \begin{tabular}{lS[table-format=2.2(2.2), table-alignment-mode=none, table-number-alignment=left]S[table-format=2.2(2.2), table-alignment-mode=none, table-align-uncertainty]}
                    \toprule
                    \textbf{Method} & {\textbf{PSNR \(\uparrow\) (dB)}} & {\textbf{RFE \(\downarrow\) (\%)}} \\ 
                    \midrule
                    INR InSitu-Subsample & 21.4(1.4) & 49.6(6.7) \\
                    INR InSitu-FJLT & 27.1(1.4) & 22.6(8.1) \\ 
                    Hypernet + INR InSitu-Subsample & 28.2(9.3) & 32.0(27.5) \\
                    Hypernet + INR InSitu-FJLT & 37.2(0.7) & 4.7(0.5) \\
                    \bottomrule
                \end{tabular}
                \caption{Performance comparison with other \textit{in situ} INR-based methods on the Ignition dataset at approximately 140\(\times\) compression rate with a \(0.7\)\% sample factor. Metric statistics are reported as mean plus-or-minus standard deviation from five identical trials.}
                \label{tab:inr-comparison}
            \end{table}

    \section{Discussion and Future Work}\label{sec:conclusion}

        We have presented a new method for training neural compressors \textit{in situ} with a simulation. The method primarily uses a limited memory buffer of sketched data samples to regularize against catastrophic forgetting during the learning process. This approach is theoretically motivated via a Johnson-Lindenstrauss-type result and empirically validated by showing that offline neural compression results can be matched in the \textit{in situ} setting. In our particular case, we employ a hypernetwork paired with an implicit neural representation (INR), but the training approach can be extended to other network architectures. While our results are encouraging, a myriad of relevant questions remain unanswered that deserve further exploration. We break this into the two major components of our work: sketching for \textit{in situ} learning and neural compression.

        \paragraph{Sketching for \textit{In Situ} Learning} Our investigation herein has only scratched the surface of how this paradigm could be used. To begin, the sketching techniques we considered --- subsampling and the fast Johnson-Lindenstrauss transform (FJLT) --- hardly comprise a complete overview of available options. Research on the best sketching method is warranted, including deterministic sketches or sketches that are data-informed in some regard.
        
        Beyond the operator used to sketch, the protocol for when to sketch and how to incorporate them into the buffer has significant room for improvement. So far, we have considered identically sized sketches for each snapshot, but this is simply the most straightforward method to use the available memory. For example, it may be desirable to allocate more storage, allowing for larger sketches, to earlier snapshots. Additionally, an adaptive method for overwriting past sketch data would allow one to use all available memory at all times. As for the full snapshot buffer, we only considered the case where a single sample is available at a time. In practice, this number could be far larger, providing opportunities for sketching or encoding as a block.

        \revb{In a very similar vein, there is significant room for improvement in determining the sketch regularization weight. While it is motivating that our simply chosen fixed value performed well across a variety of experiments, a more informed approach is likely to yield superior results. Furthermore, this coefficient provides a means for weighting sketches differently depending on desirable characteristics. As a simple example, one may want to weight older sketches more heavily, to regularize more aggressively against snapshots the network has had more time to ``forget''.}

        On the theoretical side, much more can be done to explain the empirical success of small sketches in preventing catastrophic forgetting. The relevant question to ask is how large a sketch needs to be to serve as an effective regularizer? Rigorously answering this question would lead to a more precise setting of the sketch size. In a similar vein, our use of hypernetworks was empirically motivated, but deserves a more thorough examination as to why it works better with this regularization technique than a standard implicit neural representation (INR).

        \paragraph{Neural Compression} Further work remains to make neural compression into a practical tool. While some improvements are straightforward, they have been set aside here to keep the focus on our primary contribution, i.e., demonstrating \textit{in situ} training on an INR. For example, post-training (or other more complex approaches) network quantization can significantly increase the compression rate with little impact on the reconstruction performance. As well, we do not incorporate spatial gradients into our loss, but that is often an important aspect of sine-based INR (SIREN). On a similar note, we do not enforce any physics, despite having access to such knowledge from the simulation. Taking advantage of this extra information could improve the overall performance. In a more general sense, we also do not optimize the model architecture or its learning process.

        In practice, many, if not most, large-scale simulations are carried out in a distributed manner, where the mesh is split into several similarly sized partitions. It is unclear how this change in setting would impact the neural compression approach. For example, one could optimize a central network via updates from the distributed partitions, or one could optimize distributed networks and then combine them in a distillation phase.
    
    \section*{Acknowledgments}\label{sec:acknowledgments}

        This material is based upon work supported by the Department of Energy, National Nuclear Security Administration under Award Number DE-NA0003968, as well as Department of Energy Advanced Scientific Computing Research Awards DE-SC0022283 and DE-SC0023346. Alireza Doostan's work has also been partially supported by the AFOSR grant FA9550-20-1-0138. Cooper Simpson's work has also been supported by the U.S. Department of Energy, Office of Science, Office of Advanced Scientific Computing Research, Department of Energy Computational Science Graduate Fellowship under Award Number DE-SC0026073.
        
        We would also like to thank Kenneth Jansen, John Evans, Kevin Doherty, Angran Li, and Jeff Hadley from the University of Colorado Boulder for their helpful discussions surrounding this work. We also thank Youngkyu Kim for performing some insightful experiments regarding the performance of various sketches.
        
        Additionally, this work was facilitated by several high-performance computing resources:
        \begin{itemize}
            \item The Alpine high-performance computing resource at the University of Colorado Boulder. Alpine is jointly funded by the University of Colorado Boulder, the University of Colorado Anschutz, and Colorado State University, and with support from NSF grants OAC-2201538 and OAC-2322260.
            \item The Blanca condo computing resource at the University of Colorado Boulder. Blanca is jointly funded by computing users and the University of Colorado Boulder.
            \item Advanced computational, storage, and networking infrastructure provided by the Hyak supercomputer system and funded by the STF at the University of Washington.
        \end{itemize}

        This report was prepared as an account of work sponsored by an agency of the United States Government. Neither the United States Government nor any agency thereof, nor any of their employees, makes any warranty, express or implied, or assumes any legal liability or responsibility for the accuracy, completeness, or usefulness of any information, apparatus, product, or process disclosed, or represents that its use would not infringe privately owned rights. Reference herein to any specific commercial product, process, or service by trade name, trademark, manufacturer, or otherwise does not necessarily constitute or imply its endorsement, recommendation, or favoring by the United States Government or any agency thereof. The views and opinions of authors expressed herein do not necessarily state or reflect those of the United States Government or any agency thereof.
    
    \newpage
    \bibliographystyle{unsrtnat}
    \bibliography{ref}

@article{pacella2022task,
  title={Task-parallel in situ temporal compression of large-scale computational fluid dynamics data},
  author={Pacella, Heather and Dunton, Alec and Doostan, Alireza and Iaccarino, Gianluca},
  journal={The International Journal of High Performance Computing Applications},
  volume={36},
  number={3},
  pages={388--418},
  year={2022},
  publisher={SAGE Publications Sage UK: London, England}
}

@article{halko2011finding,
  title={Finding structure with randomness: Probabilistic algorithms for constructing approximate matrix decompositions},
  author={Halko, Nathan and Martinsson, Per-Gunnar and Tropp, Joel A},
  journal={SIAM review},
  volume={53},
  number={2},
  pages={217--288},
  year={2011},
  publisher={SIAM}
}

@article{dunton2020pass,
  title={Pass-efficient methods for compression of high-dimensional turbulent flow data},
  author={Dunton, Alec M and Jofre, Llu{\'\i}s and Iaccarino, Gianluca and Doostan, Alireza},
  journal={Journal of Computational Physics},
  volume={423},
  pages={109704},
  year={2020},
  publisher={Elsevier}
}

@article{caruana1997multitask,
    title={Multitask learning},
    author={Caruana, Rich},
    journal={Machine learning},
    volume={28},
    number={1},
    pages={41--75},
    year={1997},
    publisher={Springer}
}

@article{kirkpatrick2017overcoming,
    title={Overcoming catastrophic forgetting in neural networks},
    author={Kirkpatrick, James and Pascanu, Razvan and Rabinowitz, Neil and Veness, Joel and Desjardins, Guillaume and Rusu, Andrei A and Milan, Kieran and Quan, John and Ramalho, Tiago and Grabska-Barwinska, Agnieszka and others},
    journal={Proceedings of the national academy of sciences},
    volume={114},
    number={13},
    pages={3521--3526},
    year={2017},
    publisher={National Academy of Sciences}
}

@inproceedings{ha2017hypernetworks,
    title={HyperNetworks},
    author={David Ha and Andrew M. Dai and Quoc V. Le},
    booktitle={International Conference on Learning Representations},
    year={2017}
}

@inproceedings{park2019deepsdf,
    title={Deepsdf: Learning continuous signed distance functions for shape representation},
    author={Park, Jeong Joon and Florence, Peter and Straub, Julian and Newcombe, Richard and Lovegrove, Steven},
    booktitle={Proceedings of the IEEE/CVF conference on computer vision and pattern recognition},
    pages={165--174},
    year={2019}
}

@inproceedings{genova2019learning,
    title={Learning shape templates with structured implicit functions},
    author={Genova, Kyle and Cole, Forrester and Vlasic, Daniel and Sarna, Aaron and Freeman, William T and Funkhouser, Thomas},
    booktitle={Proceedings of the IEEE/CVF international conference on computer vision},
    pages={7154--7164},
    year={2019}
}

@article{doherty2024quadconv,
    title={QuadConv: Quadrature-based convolutions with applications to non-uniform PDE data compression},
    author={Doherty, Kevin and Simpson, Cooper and Becker, Stephen and Doostan, Alireza},
    journal={Journal of Computational Physics},
    volume={498},
    pages={112636},
    year={2024},
    publisher={Elsevier}
}

@article{sitzmann2020implicit,
    title={Implicit neural representations with periodic activation functions},
    author={Sitzmann, Vincent and Martel, Julien and Bergman, Alexander and Lindell, David and Wetzstein, Gordon},
    journal={Advances in neural information processing systems},
    volume={33},
    pages={7462--7473},
    year={2020}
}

@inproceedings{saragadam2023wire,
    title={Wire: Wavelet implicit neural representations},
    author={Saragadam, Vishwanath and LeJeune, Daniel and Tan, Jasper and Balakrishnan, Guha and Veeraraghavan, Ashok and Baraniuk, Richard G},
    booktitle={Proceedings of the IEEE/CVF Conference on Computer Vision and Pattern Recognition},
    pages={18507--18516},
    year={2023}
}

@inproceedings{roddenberry2023implicit,
    title={Implicit Neural Representations and the Algebra of Complex Wavelets},
    author={Roddenberry, T Mitchell and Saragadam, Vishwanath and de Hoop, Maarten V and Baraniuk, Richard},
    booktitle={The Twelfth International Conference on Learning Representations},
    year={2024}
}

@inproceedings{sales2024implicit,
    title={Implicit Neural Compression for Aerospace Simulation Visualisation},
    author={Sales, Robert M and Pullan, Graham},
    booktitle={AIAA SCITECH 2024 Forum},
    pages={0164},
    year={2024}
}

@inproceedings{strumpler2022implicit,
    title={Implicit neural representations for image compression},
    author={Str{\"u}mpler, Yannick and Postels, Janis and Yang, Ren and Gool, Luc Van and Tombari, Federico},
    booktitle={European Conference on Computer Vision},
    pages={74--91},
    year={2022},
    organization={Springer}
}

@inproceedings{dupont2021coin,
    title={COIN: COmpression with Implicit Neural representations},
    author={Dupont, Emilien and Golinski, Adam and Alizadeh, Milad and Teh, Yee Whye and Doucet, Arnaud},
    booktitle={Neural Compression: From Information Theory to Applications--Workshop@ ICLR 2021},
    year={2021}
}

@inproceedings{saragadam2022miner,
    title={Miner: Multiscale implicit neural representation},
    author={Saragadam, Vishwanath and Tan, Jasper and Balakrishnan, Guha and Baraniuk, Richard G and Veeraraghavan, Ashok},
    booktitle={European Conference on Computer Vision},
    pages={318--333},
    year={2022},
    organization={Springer}
}

@inproceedings{yang2023tinc,
    title={Tinc: Tree-structured implicit neural compression},
    author={Yang, Runzhao},
    booktitle={Proceedings of the IEEE/CVF Conference on Computer Vision and Pattern Recognition},
    pages={18517--18526},
    year={2023}
}

@inproceedings{pistilli2022signal,
    title={Signal compression via neural implicit representations},
    author={Pistilli, Francesca and Valsesia, Diego and Fracastoro, Giulia and Magli, Enrico},
    booktitle={ICASSP 2022-2022 IEEE International Conference on Acoustics, Speech and Signal Processing (ICASSP)},
    pages={3733--3737},
    year={2022},
    organization={IEEE}
}

@inproceedings{zhang2021implicit,
    title={Implicit neural video compression},
    author={Zhang, Yunfan and Van Rozendaal, Ties and Brehmer, Johann and Nagel, Markus and Cohen, Taco},
    booktitle={ICLR Workshop on Deep Generative Models for Highly Structured Data},
    year={2022}
}

@inproceedings{pham2023autoencoding,
    title={Autoencoding Implicit Neural Representations for Image Compression},
    author={Pham, Tuan and Yang, Yibo and Mandt, Stephan},
    booktitle={ICML 2023 Workshop Neural Compression: From Information Theory to Applications},
    year={2023}
}

@inproceedings{lu2021compressive,
    title={Compressive neural representations of volumetric scalar fields},
    author={Lu, Yuzhe and Jiang, Kairong and Levine, Joshua A and Berger, Matthew},
    booktitle={Computer Graphics Forum},
    volume={40},
    pages={135--146},
    year={2021},
    organization={Wiley Online Library}
}

@article{chen2021nerv,
    title={Nerv: Neural representations for videos},
    author={Chen, Hao and He, Bo and Wang, Hanyu and Ren, Yixuan and Lim, Ser Nam and Shrivastava, Abhinav},
    journal={Advances in Neural Information Processing Systems},
    volume={34},
    pages={21557--21568},
    year={2021}
}

@article{wang2024comprehensive,
    title={A comprehensive survey of continual learning: Theory, method and application},
    author={Wang, Liyuan and Zhang, Xingxing and Su, Hang and Zhu, Jun},
    journal={IEEE Transactions on Pattern Analysis and Machine Intelligence},
    year={2024},
    volume={46},
    pages={5362--5383},
    publisher={IEEE}
}

@inproceedings{von2019continual,
    title={Continual learning with hypernetworks},
    author={Von Oswald, Johannes and Henning, Christian and Grewe, Benjamin F and Sacramento, Jo{\~a}o},
    booktitle={International Conference on Learning Representations},
    year={2020}
}

@article{chauhan2024brief,
    title={A brief review of hypernetworks in deep learning},
    author={Chauhan, Vinod Kumar and Zhou, Jiandong and Lu, Ping and Molaei, Soheila and Clifton, David A},
    journal={Artificial Intelligence Review},
    volume={57},
    number={9},
    pages={1--29},
    year={2024},
    publisher={Springer}
}

@article{pan2023neural,
    title={Neural implicit flow: a mesh-agnostic dimensionality reduction paradigm of spatio-temporal data},
    author={Pan, Shaowu and Brunton, Steven L and Kutz, J Nathan},
    journal={Journal of Machine Learning Research},
    volume={24},
    number={41},
    pages={1--60},
    year={2023}
}

@article{ailon2009fast,
    title={The fast Johnson--Lindenstrauss transform and approximate nearest neighbors},
    author={Ailon, Nir and Chazelle, Bernard},
    journal={SIAM Journal on computing},
    volume={39},
    number={1},
    pages={302--322},
    year={2009},
    publisher={SIAM}
}

@article{han2023kd,
    title={KD-INR: Time-varying volumetric data compression via knowledge distillation-based implicit neural representation},
    author={Han, Jun and Zheng, Hao and Bi, Chongke},
    journal={IEEE Transactions on Visualization and Computer Graphics},
    volume={30},
    pages={6826--6838},
    year={2023},
    publisher={IEEE}
}

@misc{paper-repo,
    author={Simpson, Cooper},
    title={{Implicit-Neural-Compression}},
    howpublished={\url{https://github.com/RS-Coop/Implicit-Neural-Compression}},
    note={Numerical experiments for insitu implicit neural compression in Python},
    year={2025}
}

@inproceedings{tang2024ecnr,
    title={Ecnr: efficient compressive neural representation of time-varying volumetric datasets},
    author={Tang, Kaiyuan and Wang, Chaoli},
    booktitle={2024 IEEE 17th Pacific Visualization Conference (PacificVis)},
    pages={72--81},
    year={2024},
    organization={IEEE Computer Society}
}

@article{rolnick2019experience,
    title={Experience replay for continual learning},
    author={Rolnick, David and Ahuja, Arun and Schwarz, Jonathan and Lillicrap, Timothy and Wayne, Gregory},
    journal={Advances in neural information processing systems},
    volume={32},
    year={2019}
}

@article{avron2010blendenpik,
    title={Blendenpik: Supercharging LAPACK's least-squares solver},
    author={Avron, Haim and Maymounkov, Petar and Toledo, Sivan},
    journal={SIAM Journal on Scientific Computing},
    volume={32},
    number={3},
    pages={1217--1236},
    year={2010},
    publisher={SIAM}
}

@article{huang2021spectral,
    title={Spectral estimation from simulations via sketching},
    author={Huang, Zhishen and Becker, Stephen},
    journal={Journal of Computational Physics},
    volume={447},
    pages={110686},
    year={2021},
    publisher={Elsevier}
}

@article{baraniuk2009random,
    title={Random projections of smooth manifolds},
    author={Baraniuk, Richard G and Wakin, Michael B},
    journal={Foundations of computational mathematics},
    volume={9},
    number={1},
    pages={51--77},
    year={2009},
    publisher={Springer}
}

@article{dunton2021deterministic,
    title={Deterministic matrix sketches for low-rank compression of high-dimensional simulation data},
    author={Dunton, Alec Michael and Doostan, Alireza},
    journal={arXiv preprint arXiv:2105.01271},
    year={2021}
}

@techreport{single_precision,
    title={{IEEE Standard for Binary Floating-Point Arithmetic}},
    author={IEEE},
    institution={IEEE Standards Association},
    number={ANSI/IEEE Std 754-1985},
    year={1985},
    doi={10.1109/IEEESTD.1985.82928}
}

@article{di2025survey,
    title={A survey on error-bounded lossy compression for scientific datasets},
    author={Di, Sheng and Liu, Jinyang and Zhao, Kai and Liang, Xin and Underwood, Robert and Zhang, Zhaorui and Shah, Milan and Huang, Yafan and Huang, Jiajun and Yu, Xiaodong and others},
    journal={ACM computing surveys},
    volume={57},
    number={11},
    pages={1--38},
    year={2025},
    publisher={ACM New York, NY}
}

@inproceedings{pytorch,
    author={Ansel, Jason and Yang, Edward and He, Horace and Gimelshein, Natalia and Jain, Animesh and Voznesensky, Michael and Bao, Bin and Bell, Peter and Berard, David and Burovski, Evgeni and Chauhan, Geeta and Chourdia, Anjali and Constable, Will and Desmaison, Alban and DeVito, Zachary and Ellison, Elias and Feng, Will and Gong, Jiong and Gschwind, Michael and Hirsh, Brian and Huang, Sherlock and Kalambarkar, Kshiteej and Kirsch, Laurent and Lazos, Michael and Lezcano, Mario and Liang, Yanbo and Liang, Jason and Lu, Yinghai and Luk, CK and Maher, Bert and Pan, Yunjie and Puhrsch, Christian and Reso, Matthias and Saroufim, Mark and Siraichi, Marcos Yukio and Suk, Helen and Suo, Michael and Tillet, Phil and Wang, Eikan and Wang, Xiaodong and Wen, William and Zhang, Shunting and Zhao, Xu and Zhou, Keren and Zou, Richard and Mathews, Ajit and Chanan, Gregory and Wu, Peng and Chintala, Soumith},
    booktitle={29th ACM International Conference on Architectural Support for Programming Languages and Operating Systems, Volume 2 (ASPLOS '24)},
    doi={10.1145/3620665.3640366},
    month=apr,
    publisher={ACM},
    title={{PyTorch 2: Faster Machine Learning Through Dynamic Python Bytecode Transformation and Graph Compilation}},
    year={2024}
}

@article{channel_flow_1,
    title={A web services accessible database of turbulent channel flow and its use for testing a new integral wall model for LES},
    author={Graham, J and Kanov, K and Yang, XIA and Lee, M and Malaya, N and Lalescu, CC and Burns, R and Eyink, G and Szalay, A and Moser, RD and others},
    journal={Journal of Turbulence},
    volume={17},
    number={2},
    pages={181--215},
    year={2016},
    publisher={Taylor \& Francis}
}

@article{channel_flow_2,
    title={A public turbulence database cluster and applications to study Lagrangian evolution of velocity increments in turbulence},
    author={Li, Yi and Perlman, Eric and Wan, Minping and Yang, Yunke and Meneveau, Charles and Burns, Randal and Chen, Shiyi and Szalay, Alexander and Eyink, Gregory},
    journal={Journal of Turbulence},
    volume={9},
    pages={N31},
    year={2008},
    publisher={Taylor \& Francis}
}

@article{lee2020model,
  title={Model reduction of dynamical systems on nonlinear manifolds using deep convolutional autoencoders},
  author={Lee, Kookjin and Carlberg, Kevin T},
  journal={Journal of Computational Physics},
  volume={404},
  pages={108973},
  year={2020},
  publisher={Elsevier}
}

@article{essakine2024we,
    title={Where Do We Stand with Implicit Neural Representations? A Technical and Performance Survey},
    author={Essakine, Amer and Cheng, Yanqi and Cheng, Chun-Wun and Zhang, Lipei and Deng, Zhongying and Zhu, Lei and Sch{\"o}nlieb, Carola-Bibiane and Aviles-Rivero, Angelica I},
    journal={Transactions on Machine Learning Research},
    issn={2835-8856},
    year={2025}
}

@article{doherty2024mfz,
    title = {Mesh-Float-Zip (MFZ): Manifold harmonic bases for unstructured spatial data compression},
    journal = {Applied Mathematics for Modern Challenges},
    volume = {2},
    number = {4},
    pages = {465-489},
    year = {2024},
    doi = {10.3934/ammc.2024023},
    author = {Kevin Doherty and Stephen Becker and Alireza Doostan}
}

@article{lindstrom2014zfp,
    author={Lindstrom, Peter},
    journal={IEEE Transactions on Visualization and Computer Graphics}, 
    title={Fixed-Rate Compressed Floating-Point Arrays}, 
    year={2014},
    volume={20},
    number={12},
    pages={2674-2683},
    doi={10.1109/TVCG.2014.2346458}
}

@article{liang2023sz,
    author={Liang, Xin and Zhao, Kai and Di, Sheng and Li, Sihuan and Underwood, Robert and Gok, Ali M. and Tian, Jiannan and Deng, Junjing and Calhoun, Jon C. and Tao, Dingwen and Chen, Zizhong and Cappello, Franck},
    journal={IEEE Transactions on Big Data}, 
    title={SZ3: A Modular Framework for Composing Prediction-Based Error-Bounded Lossy Compressors}, 
    year={2023},
    volume={9},
    number={2},
    pages={485-498},
    doi={10.1109/TBDATA.2022.3201176}
}

@article{madireddy2021,
    doi={10.1088/2632-2153/abc326},
    year={2020},
    month={12},
    publisher={IOP Publishing},
    volume={2},
    number={2},
    pages={025010},
    author={Madireddy, Sandeep and Hwan Park, Ji and Lee, Sunwoo and Balaprakash, Prasanna and Yoo, Shinjae and Liao, Wei-keng and Hauck, Cory D and Paul Laiu, M and Archibald, Richard},
    title={In situ compression artifact removal in scientific data using deep transfer learning and experience replay},
    journal={Machine Learning: Science and Technology}
}

@article{raissi2019physics,
    title={Physics-informed neural networks: A deep learning framework for solving forward and inverse problems involving nonlinear partial differential equations},
    author={Raissi, Maziar and Perdikaris, Paris and Karniadakis, George E.},
    journal={Journal of Computational physics},
    volume={378},
    pages={686--707},
    year={2019},
    publisher={Elsevier}
}

@article{johnson1984extensions,
    title={Extensions of {L}ipschitz mappings into a {H}ilbert space},
    author={Johnson, William B and Lindenstrauss, Joram},
    journal={Contemporary mathematics},
    volume={26},
    number={189-206},
    pages={1},
    year={1984}
}

@article{woodruff2014sketching,
    title={Sketching as a tool for numerical linear algebra},
    author={Woodruff, David P.},
    journal={Foundations and Trends{\textregistered} in Theoretical Computer Science},
    volume={10},
    number={1--2},
    pages={1--157},
    year={2014},
    publisher={Now Publishers, Inc.}
}

@article{ArminMike2015,
    title = {New analysis of manifold embeddings and signal recovery from compressive measurements},
    journal = {Applied and Computational Harmonic Analysis},
    volume = {39},
    number = {1},
    pages = {67-109},
    year = {2015},
    issn = {1063-5203},
    doi = {https://doi.org/10.1016/j.acha.2014.08.005},
    author = {Armin Eftekhari and Michael B. Wakin}
}

@book{Lee_manifoldBook2ndEd,
    title={Introduction to Smooth manifolds, 2nd ed},
    author={Lee, John},
    edition = {2},
    year={2012},
    address = {New York},
    publisher={Springer}
}

@article{Angran3D,
    title={Modeling intracellular transport and traffic jam in {3D} neurons using {PDE}-constrained optimization},
    author={Li, Angran and Zhang, Yongjie Jessica},
    journal={Journal of Mechanics},
    volume={38},
    pages={44--59},
    year={2022},
    publisher={Oxford University Press}
}

@article{li2025online,
    title={Online randomized interpolative decomposition with a posteriori error estimator for temporal PDE data reduction},
    author={Li, Angran and Becker, Stephen and Doostan, Alireza},
    journal={Computer Methods in Applied Mechanics and Engineering},
    volume={434},
    pages={117538},
    year={2025},
    publisher={Elsevier}
}

@inproceedings{yu2017single,
    title={Single-pass PCA of large high-dimensional data},
    author={Yu, Wenjian and Gu, Yu and Li, Jian},
    booktitle={Proceedings of the 26th International Joint Conference on Artificial Intelligence},
    pages={3350--3356},
    year={2017}
}

@article{eckart1936approximation,
    title={The approximation of one matrix by another of lower rank},
    author={Eckart, Carl and Young, Gale},
    journal={Psychometrika},
    volume={1},
    number={3},
    pages={211--218},
    year={1936},
    publisher={Springer-Verlag}
}

@misc{scikit-dimension,
    author={Jonathan Bac},
    title={scikit-dimension, ver.\ 0.3.4},
    howpublished={\url{https://scikit-dimension.readthedocs.io/}},
    year={2021}
}

@article{fukunaga1971algorithm,
    title={An algorithm for finding intrinsic dimensionality of data},
    author={Fukunaga, Keinosuke and Olsen, David R},
    journal={IEEE Transactions on computers},
    volume={100},
    number={2},
    pages={176--183},
    year={1971},
    publisher={IEEE}
}

@article{zfp-software,
    author={Diffenderfer, James and Fox, Alyson L. and Hittinger, Jeffrey A. and Sanders, Geoffrey and Lindstrom, Peter G.},
    title={Error Analysis of ZFP Compression for Floating-Point Data},
    journal={SIAM Journal on Scientific Computing},
    volume={41},
    number={3},
    pages={A1867-A1898},
    year={2019},
    doi={10.1137/18M1168832}
}



\end{document}